\PassOptionsToPackage{usenames,dvipsnames}{xcolor}
\documentclass[11pt, letterpaper]{article}

\usepackage{times}

\usepackage{soul}
\usepackage{url}
\usepackage{breakurl}
\usepackage[hidelinks]{hyperref}
\usepackage[utf8]{inputenc}
\usepackage{graphicx}
\usepackage{amsmath}
\usepackage{booktabs}
\usepackage{enumitem}

\usepackage{fullpage}
\usepackage{amsthm}
\usepackage{amssymb}
\usepackage{algorithm}
\usepackage[noend]{algorithmic}
\usepackage{todonotes}
\usepackage{xcolor}
\usepackage{subcaption}

\newtheorem{lemma}{Lemma}
\newtheorem{observation}{Observation}

\newtheorem{hypothesis}{Hypothesis}

\newcommand{\bos}{\texttt{BOSTON}}
\newcommand{\nh}{\texttt{NHANES}}
\newcommand{\fl}{\texttt{FLIGHTS}}
\newcommand{\hi}{\texttt{HEALTH\_INSURANCE}}
\newcommand{\synd}{\texttt{SYNTHETIC\_DENSE}}
\newcommand{\syns}{\texttt{SYNTHETIC\_SPARSE}}
\newcommand{\bosgb}{\texttt{BOSTON\_GBDT}}
\newcommand{\nhgb}{\texttt{NHANES\_GBDT}}
\newcommand{\flgb}{\texttt{FLIGHTS\_GBDT}}
\newcommand{\higb}{\texttt{HEALTH\_INSURANCE\_GBDT}}
\newcommand{\bosdt}{\texttt{BOSTON\_DT}}
\newcommand{\nhdt}{\texttt{NHANES\_DT}}
\newcommand{\fldt}{\texttt{FLIGHTS\_DT}}
\newcommand{\hidt}{\texttt{HEALTH\_INSURANCE\_DT}}
\newcommand{\shapours}{\texttt{shap\_orig\_a}}
\newcommand{\shapfast}{\texttt{shap\_fast}}
\newcommand{\shaporig}{\texttt{shap\_orig}}
\newcommand{\ban}{\texttt{ban}}

\newcommand{\tr}{\mathcal{T}}
\newcommand{\lvs}{\mathcal{L}}

\def\twofigs #1{\hbox to \textwidth{#1}}

\newcommand{\EX}{\mathbb{E}}

\begin{document}

  \title{Improved Feature Importance Computations for Tree Models: Shapley vs. Banzhaf\thanks{All authors were supported by the ERC Consolidator Grant 772346 TUgbOAT.}}

  \author{
Adam Karczmarz\thanks{University of Warsaw, \tt{a.karczmarz@mimuw.edu.pl}}
\and
Anish Mukherjee\thanks{University of Warsaw, \tt{anish@@mimuw.edu.pl}}
\and
Piotr Sankowski\thanks{IDEAS NCBR, University of Warsaw, and MIM Solutions, Warsaw, \tt{sank@mimuw.edu.pl}}
\and
Piotr Wygocki\thanks{University of Warsaw, and MIM Solutions, Warsaw, {\tt{wygos@mimuw.edu.pl}}. Additionally supported by the Polish National Science Center Grant PRELUDIM 2018/29/N/ST6/00676.}
}

  \date{}

\maketitle
  \thispagestyle{empty}
\begin{abstract}
Shapley values are one of the main tools used to explain predictions of tree ensemble models. The main alternative to Shapley values are Banzhaf values that have not been understood equally well. In this paper we make a step towards filling this gap, providing both experimental and theoretical comparison of these model explanation methods. Surprisingly, we show that Banzhaf values offer several advantages over Shapley values while providing essentially the same explanations. We verify that Banzhaf values:
\begin{itemize}
\item have a more intuitive interpretation,
\item allow for more efficient algorithms,
\item are much more numerically robust.
\end{itemize}
We provide an experimental evaluation of these theses. In particular, we show that on real world instances.

Additionally, from a theoretical perspective we provide new and improved algorithm computing the same Shapley value based explanations as the algorithm of Lundberg et al. [Nat. Mach. Intell. 2020]. Our algorithm runs in $O(TLD+n)$ time, whereas the previous algorithm had $O(TLD^2+n)$ running time bound. Here, $T$ is the number of trees, $L$ is the maximum number of leaves in a tree, and $D$ denotes the maximum depth of a tree in the ensemble. Using the computational techniques developed for Shapley values we deliver an optimal $O(TL+n)$ time algorithm for computing Banzhaf values based explanations. In our experiments these algorithms give running times smaller even by an order of magnitude.
\end{abstract}

\clearpage
\setcounter{page}{1}

\section{Introduction}
The explainability of machine learning models has become one of the crucial aspects when deploying such models in practice. When
high-value decisions are taken, understanding why a model made a certain prediction is even more important than the
prediction's accuracy. In such applications, e.g., medical diagnostic, taking into account the necessity of human oversight
is a must. Thus we need to deliver methods that would interpret the model's results, so that humans are actually willing to follow model recommendations.

As a result, recently there has been a growing interest in feature attribution problems,
where one would like to distribute the prediction of a model $f$ to the individual features used in the model.
The feature attributions are used to explain the relative influence
of individual features to the model's prediction $f(x)$
on some specific input~$x$.

\paragraph{Feature influence.} The most popular approaches to interpreting model predictions
is based on so-called \emph{Shapley values}~(e.g.,~\cite{Lundberg2020,Lundberg2017,StrumbeljK14,SundararajanTY17}).
The attractiveness of this approach comes from the fact that Shapley values can be very efficiently computed in
the case of tree ensemble models. Although some papers suggest using Banzhaf
values~\cite{datta2,Sliwinski_Strobel_Zick_2019,patel2020high} in place of Shapley values, this alternative has not been understood equally well.

In general, the motivation for using the Shapley/Banzhaf values in this context comes from game theory.
Their use is justified by the fact that they are known to be the unique method
of measuring the importance, or value, of a player in a coalitional game $g:2^U\to\mathbb{R}$,
that satisfies a certain set of desirable axioms which differ slightly for Shapley and Banzhaf values. Here, $g$ gives an utility of a coalition,
$g(\emptyset)=0$ and $U$ is the set of players (or -- in our context -- features).

However, as explained in~\cite{SundararajanN20}, despite the uniqueness of Shapley values,
many different Shapley values-based attribution methods have been studied.
This is because, we need to define a \emph{set function} $g$
that extends $f$ to all subset of features $S\subseteq U$, i.e.,
$g$ allows us to drop features $U\setminus S$ of $x$. Given such a set function $g$, and $n=|U|$,
the Shapley value of the feature $i\in U$ is defined:
\begin{equation}\label{eq:shap}
  \phi_i=\frac{1}{n}\sum_{S\subseteq U\setminus\{i\}} {\binom{n-1}{|S|}}^{-1}\left(g(S\cup\{i\})-g(S)\right).
\end{equation}

Let us mention some concrete examples of $g$ used in~\cite{JanzingMB20,Lundberg2020,SundararajanN20}.
In the \emph{Baseline Shapley} approach (BShap), some baseline feature vector
${x'\in \mathbb{R}^U}$ is given. Then for any $S\subseteq U$
${g(S):=f(x_S,x'_{U\setminus S})-f(x'_U)}$, where $f(x_S,x'_S)$ is the output of the model
for a feature vector whose feature values for features in $S$ are taken
from $x$, whereas the values for features in $U\setminus S$ are taken
from $x'$.

In the \emph{Marginal Expectation Shapley}
approach (MES), $g(S):=\EX[f(x_S,X_{U\setminus S})]-\EX[f(X_U)]$,
where $X_{V}$ denote the random variables corresponding
to the values of features in $V\subseteq U$.
Similarly, in the \emph{Conditional Expectation Shapley} approach (CES),
$g(S):=\EX[f(X_U)|X_S=x_S]-\EX[f(X_U)]$.

We note that explanations based on Shapley values have been extensively studied experimentally~\cite{Lundberg2020,Lundberg2017,StrumbeljK14,SundararajanTY17}, whereas
in the case of Banzhaf values such study was only done on single data sets~\cite{datta2,Sliwinski_Strobel_Zick_2019,patel2020high}. Moreover, despite very high similarity
of both methods, only limited experimental comparison between them can be found in the literature. We are only
aware of the recent comparison of the both methods in~\cite{patel2020high}, which is limited to a single depth-3 tree.
We review other methods in Section~\ref{section:other}.

This paper aims to make a step towards filling this gap by providing theoretical as well as experimental
commissional study of these methods. In particular, our contribution is the first to clearly visualize that for tree ensemble models, Banzhaf values:
\begin{itemize}
\item have a more intuitive interpretation -- they correspond directly to the probability that a model prediction changes when changing a feature,
\item allow for more efficient algorithms -- in our experiments the speed-up is by an order of magnitude,
\item are much more numerically robust -- we observe that due to numerical errors, using Shapley values can lead to a wrong ordering of features.
\end{itemize}
Thus our paper adds new cardinal arguments towards usage of Banzhaf values for feature attribution.

\paragraph{Computational complexity.}
In general, computing Shapley and Banzhaf values is a computationally
hard problem~\cite{DengP94,Matsui2001}.
With just a black-box access to the set function, it seems
that a naive evaluation of the formula~(\ref{eq:shap}) is the best
we can hope for. However, for some specific models $f$ and choices of the set
function, Shapley values can be computed much faster.

One of the main contributions of Lundberg et al.~\cite{Lundberg2020} was to show efficient
algorithms computing Shapley values for binary tree ensemble models
for two different plausible set functions $g$ (see Section~\ref{s:prelims}).
One of these algorithms, called \texttt{TREESHAP\_PATH}~\cite[Algorithm 2]{Lundberg2020},
is particularly interesting and convenient;
it provides explanations without the help of the background
dataset.
Instead, it uses a set function based on subtree coverages (see Algorithm~\ref{alg:est})
proposed in the classical work Friedman~\cite{friedman2001} for estimating
partial dependence plots.
The algorithm runs in ${O(TLD^2+n)}$ time,
where $T$ is the number of trees, $L$ is the maximum number of leaves in a tree,
and $D$ denotes the maximum depth of a tree in the ensemble.
Note also that using Equation~\eqref{eq:shap} naively would take exponential $\Omega(TnL\cdot 2^n)$ time in this case.
As noted in~\cite{patel2020high}, the same algorithm applies to Banzhaf values.

Our first contribution is a new algorithm computing the same explanation
of a tree ensemble model prediction
as the \texttt{TREESHAP\_PATH} algorithm of Lundberg et al.~\cite{Lundberg2020}.
The algorithm is asymptotically even faster and runs in $O(TLD+n)$ time.
We demonstrate that for trees of depth apx. 20 it outperforms
the original algorithm. Applying the same techniques to Banzhaf values
we are able to obtain even faster algorithm running in $O(TL+n)$ time.

\subsection{Axioms, Efficiency, and Power Indices}
As already mentioned,
the use of Shapley values for attribution is a must if we want the attribution
to satisfy some natural axioms: \emph{Efficiency}, \emph{Sensitivity}, \emph{Linearity},
and \emph{Symmetry}. For their precise definition\footnote{\emph{Sensitivity} axiom is sometimes called the \emph{dummy/null player} axiom, whereas
\emph{efficiency} is also called \emph{completeness}.}, consult~\cite{JanzingMB20,SundararajanN20}.
The efficiency axiom is particularly interesting.
It says that the sum of individual attributions $\sum_{i\in U}\phi_i$
should equal exactly $g(U)$.
In all the known approaches for defining $g$ (cf.~\cite{SundararajanN20}),
$g(U)$ equals
precisely the difference between the prediction and some baseline.
As argued in~\cite{SundararajanTY17}, efficiency
\emph{``is a sanity check that the attribution method is somewhat comprehensive in its accounting, a property that is clearly desirable if the score is used in a numeric sense, and not just to pick the top label, for e.g., a model estimating insurance premiums from credit features of individuals''}.
However, the cost of this mathematical property is a rather nonintuitive meaning of these values. In order to guarantee efficiency, one needs to weight
the contribution of each feature with the total number of orderings of the present as well as the absent features as given by~\eqref{eq:shap}. It is rather unclear
why feature vector $(man,40)$ is different from $(40,man)$, and why this should matter for feature importance.

This nonintuitive aspect becomes clearly visible when moving out of the cost-sharing regime to the study of voting power indices. Applying this concept
to model explanation, power indices essentially measure the voting power of a particular feature on the decision taken by the model.
There are several options for power indices with two being dominating ones: the Shapley-Shubik power index and the Banzhaf power index.
In some cases, Banzhaf index works better~\cite{doi:10.1111/1467-9248.00356} whereas in others Shapley-Shubik \cite{RePEc:ema:worpap:2007-25}.
Shapley-Shubik index uses Shapley values~\eqref{eq:shap}, whereas Banzhaf index
attributes \emph{Banzhaf values}~$\beta_i$ defined:
\begin{equation}\label{eq:ban}
  \beta_i=\frac{1}{2^{n-1}}\sum_{S\subseteq U\setminus\{i\}}g(S\cup \{i\})-g(S).
\end{equation}
Different axiomatic parameterizations have been given for both concepts, see e.g.~\cite{RePEc:inm:ormoor:v:26:y:2001:i:1:p:89-104}, and different paradoxes (i.e., nonintuitive
properties) were found~\cite{Koczy2009}.
Compared to Shapley values, Banzhaf values satisfy \emph{Sensitivity}, \emph{Linearity}, and \emph{Symmetry},
but not \emph{Efficiency}. This axiom is replaced by the \emph{2-Efficiency} axiom~\cite{Lehrer1988}. When studying power indices, the efficiency axiom is not considered
as a must which is the case in cost-sharing games~\cite{Hart1989}. This axiomatic approach has been followed
in~\cite{Lundberg2017,datta2,Sliwinski_Strobel_Zick_2019,patel2020high}, and analogous axiomatizations have been proven for feature values based on these concepts.

We note, however, that in the case of power-indices the discussion on different options is rather complex~\cite{power}.
In general, one points out that aspects other than existence of axioms should be used to decide which power index to use~\cite{kurz}.
Moreover, it is argued that ``axiomatic approach by itself is insufficient to settle the question of the
choice of a power index''~\cite{RePEc:ivi:wpasad:1999-10}, i.e., the consequences of different axioms can be rather non-intuitive and axioms
that seem to be the most basic ones can lead to paradoxes~\cite{Koczy2009}.
Laruelle~\cite{RePEc:ivi:wpasad:1999-10} suggests that in order to choose between Banzhaf and Shapley
probabilistic approach could be followed. In this interpretation,
Shapley index can be used when the order of players/features matters, whereas Banzhaf values should be used when this order is not important.
Nevertheless, it is observed that both methods deliver very similar results, and one usually computes both when doing experimental evaluations --
as in \cite{doi:10.1111/1467-9248.00356, RePEc:ema:worpap:2007-25}. 
Moreover, it has been formally proven that Banzhaf and Shapley-Shubik indices are
consistent when applied to restricted voting systems~\cite{MomoKenfack2019}, i.e., return the same ordering of players.
These observations motivated us to formulate the following hypothesis.
\begin{hypothesis}
\label{hyp-1}
Banzhaf and Shapley values lead to the same ordering of features for tree ensemble models.
\end{hypothesis}
We note that a priori it is unclear whether the above could be true, as
the above observations hold for voting systems which satisfy much more properties than decision models considered here.

\subsection{Banzhaf Values vs Shapley Values}
The main contribution of this paper is to fill in the experimental gap, by providing study of Shapley and Banzhaf values on tree ensemble models. In particular,
we experimentally verify Hypothesis~\ref{hyp-1} and show that both methods deliver essentially the same average importance scores for the four studied datasets.
Moreover, we prove that for monotone functions on a hypercube both measure give the same ordering.
Hence, differences in axiomatic definitions do not seem to actually deliver the answer on which index to use. Even though Banzhaf values do not obey the efficiency axiom, we argue that
they have several important advantages over Shapley values when applied to feature importance measurement.
\paragraph{Intuitive meaning.}
As pointed out in the power index literature~\cite{10.2307/3689345}, many researchers believe that the main advantage of Banzhaf values
is their more intuitive definition.
In the voting setting, the Banzhaf power index is defined as the {\em expected change to the model output when
a given vote is added}. In order to define the Shapley-Shubik power index, one considers
an ordered process of forming voter lists, i.e., one adds players one by one to the ordered list and computes
how often a given voter changes the outcome. When reinterpreting the Banzhaf power index in the machine learning
model setting, one essentially
obtains that a Banzhaf value corresponds to the expected change of the model prediction when a given feature is added.
As argued above, using Shapley values does not allow for such a straightforward interpretation.
\paragraph{Computational efficiency.}
As for Shapley values, in general, computing Banzhaf values
requires exponential time~\cite{Kislaya90}.
However, we show that Banzhaf values are tailored to
work with tree ensemble models even better than Shapley values.
The \texttt{TREESHAP\_PATH} explanation algorithm
of Lundberg~\cite{Lundberg2020}, adjusted to compute
Banzhaf values instead (when using the same natural set function $g$)
has running time $O(TLD+n)$, which constitutes a factor-$D$
speedup over the Shapley-based version.
More importantly, with the new ideas that we use
to speed-up the \texttt{TREESHAP\_PATH} algorithm,
we are able to obtain an $O(TL+n)$ time
algorithm for computing Banzhaf values for this choice of $g$.
Note that this algorithm is \emph{asymptotically optimal}, since
even the description of a tree ensemble has size $O(TL)$, and the output has size $O(n)$.
In experiments, this algorithm visibly outperforms all other algorithms
based on that set function, and
can lead to considerable time saving when computing feature importance values for decision tree-based models in practice.

We note that other Shapley values-based algorithm for trees proposed in~\cite{Lundberg2020} (\texttt{TREESHAP\_INT})
can be also easily modified to yield Banzhaf values instead within the same time bound.

\paragraph{Numerical accuracy.}
While experimenting with different implementations of Shapley value-based explanations for tree ensemble models,
we have observed that the results obtained
using different implementations are different.  These observations are discussed in Section~\ref{section:numerical}. We
note that an implementation that would use single numerical precision would be essentially useless, whereas
for double precision differences in ordering of features can be seen for trees of depth~4 (\nhgb{} instance, the difference in ordering within the first 10 features). Moreover,
we provide a very simple synthetic example of non-balanced trees showing that for depth apx. 50 numerical
errors would dominate the result. Trees of even higher depth are sometimes used in practice~\cite{s20216075}.

Banzhaf values, due to a simpler definition,
are here much more robust and only very small errors are visible in the computations. Essentially, for
trees considered in our experiments, numerical errors are not observable. This is further illustrated on
the synthetic example where numerical errors of Banzhaf method are negligible.

\paragraph{Approximate equivalency.}
In our experimental evaluation, we show that in practice the Banzhaf and Shapely values give approximately the same feature importance.
For three instances, the ordering of the mean absolute values for the top~20 most important features was the same (Figures \ref{fig:boston},~\ref{fig:health},~\ref{fig:nhanes}).
Usually, the ordering of features is very similar, e.g., for three datasets it is enough to perform a single swap on average on the first 10 features to change Shapley-based feature ordering to Banzhaf-based ordering (see Table \ref{tab:cayley}). In Section~\ref{s:detailed}, we present a comparison for mean average error and root mean square error for measuring the distance between Shapley values and Banzhaf values. Moreover, we show the worst-case comparison.

In our opinion, these arguments indicate that for tree models, Banzhaf values should be preferred in practical applications. Both methods deliver comparable explanations, but Banzhaf values work faster and
are much less prone to numerical errors.

From the theoretical perspective, we consider monotone functions on a hypercube, i.e., functions that for a given feature $i$ and any feature vector $x$, setting $x_i := 1$ (similarly $x_i := 0$) is either always increasing (or always decreasing) our value function. We prove that for such functions  and uniform distribution of the dataset, any power index in the form:
\[\sum_{S\subseteq U\setminus\{i\}}w(S)\left(g(S\cup \{i\})-g(S)\right),\]
for a distribution $w:2^U\to\mathbb{R}$, $\sum_{S\subseteq U}w(S)=1$, gives exactly the same mean absolute importance of features. In particular, this shows
that for some natural instances using Banzhaf and Shapley values
for computing global feature importance is equivalent.
See Section~\ref{monotone} for more details.

\section{Preliminaries}\label{s:prelims}
Let $U:=\{1,\ldots,n\}$ be a set of \emph{features}.
Let $x$ generally denote a \emph{feature vector}, i.e.,
$x$ is formally a function $x:U\to \mathbb{R}$.
For $i\in U$, we write $x_i$ to refer to the \emph{value}
of $i$-th feature in $x$.
More generally, for any subset $S\subseteq U$ we write
$x_S$ when referring to the function $\left.x\right|_S$.
We sometimes talk about random feature vectors, or consider
the values of individual features as random variables.
We then write $X$ or $X_i$ respectively.
We write $X_S$ to denote the set of random variables $\{X_i:i\in S\}$.
Let $\bar{S}$ denote the complement $U\setminus S$.
Let $f:\mathbb{R}^U\to \mathbb{R}$ be the output function
of our model.

\paragraph{Trees.} We focus on tree ensemble models, where the output of the model
is simply the average output of its $T$ individual trees.
For simplicity, we assume $T=1$ but still
incorporate the variable $T$ when stating the time complexities (as done in~\cite{Lundberg2020}):
the algorithms we discuss run independent computations
for all the trees in the ensemble.
Also, we assume the individual trees are roughly of the
same size and depth.

When talking about the input decision tree $\tr$, we adopt the notation of~\cite{Lundberg2020}.
$\tr$ is a binary tree based on single-variable splits:
each non-leaf node $v\in \tr$ is assigned a \emph{feature} $d_v$, and
a \emph{threshold} $t_v$, whereas each leaf $l$ is assigned a \emph{value}~$f(l)$.
Let $a_v,b_v$ denote the left and right children of a non-leaf
node $v\in \tr$. The output $f(x)$ is computed by following a root-leaf
path in $\tr$: at a non-leaf node $v\in \tr$, we descend to $a_v$ if
$x_{d_v}< t_v$, or to $b_v$ otherwise. When a leaf is reached,
its value is returned. Denote by $\lvs(\tr)$ the set of leaves of $\tr$.
Denote by $\tr[v]$ the subtree of $v$ rooted at~$v$.
Let $L$ and~$D$ denote the number of leaves and the depth
of the tree $\tr$, resp.

\paragraph{Set functions and Shapley values.}
We write $f(x_S,X_{\bar{S}})$ when referring to a random
variable defined as the value of $f$ if the
values for features in $S$ are fixed to the respective
values of $x$, and the values $X_{\bar{S}}$ are random
variables.
Let $X_U$ be distributed as in the training set.

Given a \emph{set function} $g:2^{U}\to \mathbb{R}$
with $g(\emptyset)=0$, the Shapley values $\phi_i$
for $i\in U$ are defined as in Equation~\eqref{eq:shap}.

Lundberg et al.~\cite{Lundberg2020} and Janzing et al.~\cite{JanzingMB20} suggest
using the following
idealized set function $g^*$ for feature attribution:
\begin{equation}\label{eq:g-ideal-def}
  g^*(S):=\EX[f(x_S,X_{\bar{S}})]-\EX[f(X_U)].
\end{equation}
Note how the term $\EX[f(X_U)]$ in~(\ref{eq:g-ideal-def})
cancels out when computing Shapley values from Equation~\eqref{eq:shap}.
Thus, for simplicity in the following we can redefine
${g^*(S):=\EX[f(x_S,X_{\bar{S}})]}.$

Using the idealized set function $g^*$ would be computationally
too costly.
Consequently, Lundberg et al.~\cite{Lundberg2020} considers
two different set functions that ``approximate'' $g^*$.

In the \texttt{TREESHAP\_PATH} algorithm, the approximation $g(S)\approx g^*(S)$ is computed using Algorithm~\ref{alg:est}.
This method dates back to the classical work of Friedman~\cite{friedman2001}
and is also implemented as a way to compute partial dependence plots
in the scikit-learn package~\cite{scikit-learn}.
Its one advantage is that it does not require access
to the training data, but merely to the ``coverages'' $r_v$ of all the subtrees $\tr[v]$,
i.e., the numbers of training set points that fall into $\tr[v]$.
It can be proved that this method computes $\EX[f(x_S,X_{\bar{S}})]$
if the individual feature random variables $X_i$ are independent.
With such a set function~$g$, Lundberg et al.~\cite{Lundberg2020} show how to compute
Shapley values $\phi_i$ exactly
in $O(TLD^2+n)$ time.

On the other hand, in the \texttt{TREESHAP\_INT} algorithm, Lundberg et al.~\cite{Lundberg2020}
estimate $g^*(S)$ by sampling some number $R$ of random points $x'$ of the
training data and computing the average value of $f(x_S,x'_{\bar{S}})$
over these samples.\footnote{In~\cite{SundararajanN20} this method is called \emph{Random Baseline Shapley}.}
Note that if the entire data was sampled, this would compute the
desired expectation exactly. The computation
cost would be then unacceptable, though.
The \texttt{TREESHAP\_INT} algorithm computes the Shapley values
$\phi_i$ exactly (for the described approximation of $g^*$) in $O(TRL+n)$ time.
We stress that this method requires access to the training data.

In the remaining part of the paper, we denote by $g(S)$ the
output of Algorithm~\ref{alg:est} for the subset $S\subseteq U$,
i.e., we consider the same approximation of $g^*(S)$ as
in the \texttt{TREESHAP\_PATH} algorithm of~\cite{Lundberg2020}.

\section{Improved Algorithm: Outline}
In this section we sketch an improved algorithm
computing the same Shapley value explanations
as the \texttt{TREESHAP\_PATH} algorithm
that takes
${O(TLD+n)}$ time in the worst case.
As we later show, the improvement is indeed noticeable experimentally
for unbalanced decision trees with large depth.

Let us first describe the idea behind the ${O(LD^2+n)}$ time algorithm
of Lundberg et al.~\cite{Lundberg2020} for $T=1$. To this end, we need to fix some more notation.
Let~$\rho$ be the root of $\tr$. Let~$p_v$ denote the parent of a node $v\in \tr$, $v\neq \rho$.
Let $F_v$ be the set of features assigned to the ancestors
of $v$, i.e., ${F_\rho=\emptyset}$, and $F_v=F_{p_v}\cup \{d_{p_v}\}$ for $v\neq \rho$.
The value $P[v]=r_v/r_\rho$ can be thought as the probability that the model
returns a value from the subtree of $v$.
Moreover, note that the output of Algorithm~\ref{alg:est} for $S=\emptyset$
is precisely equal to $\sum_{l\in\lvs(\tr)} P[l]\cdot f(l)$.
More generally, denote by
$P[v,S]$ the weight from the ancestor recursive calls
assigned to the subtree rooted at $v$
when running Algorithm~\ref{alg:est} with an arbitrary subset $S\subseteq U$.
Formally, $P[\rho,S]=1$, and for any $v\neq \rho$,
\begin{equation*}
  P[v,S]=\begin{cases}
      P[p_v,S]\cdot \frac{r_v}{r_{p_v}} & \text{ if } d_{p_v}\notin S,\\
      P[p_v,S]\cdot [x_{d_{p_v}}< t_{p_v}] & \text{ if }  d_{p_v}\in S, v=a_{p_v},\\
      P[p_v,S]\cdot [x_{d_{p_v}}\geq t_{p_v}] & \text{ if }  d_{p_v}\in S, v=b_{p_v}.\\
  \end{cases}
\end{equation*}

\begin{algorithm}[tb]
  \small
  \caption{Estimating $\EX[f(x_S,X_{\bar{S}})]$.}
\label{alg:est}
  \textbf{function} $\texttt{Desc}(S,v)$
\begin{algorithmic}[1]
  \IF{$v\in \lvs(\tr)$}
    \STATE \textbf{return} $f(v)$
  \ELSE
    \IF{$d_v\in S$}

    \STATE \textbf{return} (\textbf{if} $x_{d_v}< t_v$ \textbf{then } $\texttt{Desc}(S,a_v)$ \textbf{else }$\texttt{Desc}(S,b_v)$)

    \ELSE
      \STATE \textbf{return} $\frac{1}{r_v}\cdot \left(r_{a_v}\cdot \texttt{Desc}(S,a_v)+r_{b_v}\cdot \texttt{Desc}(S,b_v)\right)$
    \ENDIF
  \ENDIF
\end{algorithmic}
  \textbf{function} $g(S)$
\begin{algorithmic}[1]
  \STATE \textbf{return} $\texttt{Desc}(S,\rho)$
\end{algorithmic}
\end{algorithm}

Then, the algorithm outputs
\begin{equation}\label{eq:tpd}
  \sum_{l\in\lvs(\tr)} P[l,S]\cdot f(l)=g(S)\approx g^*(S)=\EX[f(x_S,X_{\bar{S}})].
\end{equation}

First of all, each value $\phi_i$ is obtained
by summing the contributions of each individual leaf $l\in\lvs(\tr)$
with $i\in F_l$
to the sum~(\ref{eq:shap}) with $g$ defined as in~(\ref{eq:tpd}).\footnote{If $i\notin F_l$
then it can be
shown that the contribution of $l$ to $\phi_i$ is~$0$. This is also why the \texttt{TREESHAP\_PATH} algorithm's
dependence on $n$ is just $O(n)$.}
This is in turn achieved as follows.
For any subset $G\subseteq U$, let
\begin{equation*}\label{eq:dpdef}
  \phi(v,G,k):=\frac{1}{|G|+1}\sum_{\substack{S\subseteq G\\|S|=k}} \binom{|G|}{k}^{-1}\cdot P[v,S].
\end{equation*}
Let us call the following vector of values a \emph{state} for $(v,G)$:
\begin{equation*}
  \Psi(v,G)=(\phi(v,G,k))_{k=0}^{|G|}
\end{equation*}
One can prove that moving between ``nearby'' states can be performed efficiently.
Namely, given $\Psi(v,G)$, in $O(|G|)$ time one can compute each of
the states:
$\Psi(p_v,G)$, $\Psi(a_v,G)$, $\Psi(b_v,G)$, $\Psi(v,G\cup\{i\})$, $\Psi(v,G\setminus\{i\})$,
for any feature $i\in U$.

Let $\phi(v,G)=\sum_{k=0}^{|G|}\phi(v,G,k)=||\Psi(v,G)||_1$. Clearly,
given $\Psi(v,G)$, $\phi(v,G)$ can be obtained easily in $O(|G|)$ time.
One can show that the individual contribution of leaf $l$ to $\phi_i$
for $i\in F_l$
equals
\begin{equation}\label{eq:leaf-contrib-b}
  f(l)\cdot (\phi(l,F_l)-\phi(l,F_l\setminus\{i\}))
\end{equation}

Using our notation, the algorithm of Lundberg et al. does the following.
First, all states $\Psi(v,F_v)$ for $v\in \tr$ are computed using a
simple
recursive tree traversal in time
\begin{equation*}
  O\left(\sum_{v\in \tr}|F_v|\right)=O\left(\sum_{l\in \lvs(\tr)}|F_l|\right)=O(LD),
\end{equation*}
since one can obtain $\Psi(v,F_v)$ from $\Psi(p_v,F_{p_v})$ in $O(|F_v|)=O(D)$
time.
This also gives all the values $\phi(l,F_l)$ that we need
when computing leaf contributions using~\eqref{eq:leaf-contrib-b}.
Afterwards, for each leaf $l\in \tr$, all the $|F_l|$ states
$\Psi(l,F_l\setminus\{i\})$ for $i\in F_l$ can be computed in $O(|F_l|)$ time
each. Given such a state, the leaf $l$'s contribution to $\phi_i$ can be computed in $|F_l|$ time as well.
As a result, through all pairs $(l,i)$, this takes
 $ O\left(\sum_{l\in \lvs(\tr)}|F_l|^2\right)=O(LD^2)$
time. The algorithm sketched above is described
in detail in Section~\ref{a:simple}.

The general idea behind our improved algorithm is to avoid
computing all the leaf contributions to each $\phi_i$ separately.
Instead, for a node $v\in \tr$ with $d_v=i$, we compute
the total contribution to $\phi_i$ of all leaves $\lvs_v\subseteq  \tr[v]$ for which $v$ constitutes the nearest ancestor with $d_v=i$ at once.
To this end, we show that, roughly speaking,
we can ``remove'' a feature $i$ from the sum
$\Phi(v):=\sum_{l\in\lvs_v}f(l)\cdot \phi(l,F_l)$
of all the leaf states in the subtree $\tr[v]$
(and thus obtain the sum $\sum_{l\in \lvs_v}f(l)\cdot \phi(l,F_l\setminus\{i\})$)
in $O(D)$ time as well.
This in turn allows us to get the second step of the algorithm in~\cite{Lundberg2020}
replaced with a bottom-up computation with $O(LD)$ cost.

Actually, our improvement requires some more technical care and only
works if all the sets $F_l$ have a uniform size $h$.
However, a binary tree with all sets $F_l$ having the same size~$h$
has $\Theta(2^h)$ leaves and thus padding the tree
with ``dummy'' leaves would be too costly.
To deal with this problem, we instead consider states
$\Psi(l,F_l^*)$ where $F_l^*$, $F_l\subseteq F_l^*$, is the set $F_l$ padded with $h-|F_l|$
``dummy'' features that do not appear in the tree's nodes at all.
Details can be found in Section~\ref{a:faster}.

\paragraph{Banzhaf Values for Trees}
The importance of individual features for a single prediction
using the \emph{Banzhaf values} are defined as in Equation~\eqref{eq:ban}.
Recall that we fixed $g$ to be the same approximation of $g^*$ that we used
in the previous section.
In Section~\ref{a:banzhaf} we show that for this particular set function $g$, the Banzhaf values as defined
in~\eqref{eq:ban} can be computed using an algorithm
analogous to \texttt{TREESHAP\_PATH} of~\cite{Lundberg2020} in $O(LD)$ time.
Moreover, by applying our optimization to the \texttt{TREESHAP\_PATH}
algorithm, the time bound can be further reduced to optimal $O(L)$.
Intuitively, this is possible since the coefficients of
individual terms $g(S\cup\{i\})-g(S)$ in the sum~(\ref{eq:ban})
do not depend on the size $|S|$.
As a result, an analogously defined \emph{state}
$\beta(v,G)$ consists of only one value:
\begin{equation*}
  \beta(v,G):=\frac{1}{2^{|G|}}\sum_{S\subseteq G} P[v,S],
\end{equation*}
and one can again ``move'' between the ``nearby''
states in $O(1)$ time, as opposed to $O(L)$ time
which was the case for the Shapley values computation.

\section{Monotone Functions on Hypercube}\label{monotone}
In this section we give an idealized example that sheds some light on why
Shapley and Banzhaf values can agree on some datasets. Essentially,
when the dataset gets close to the uniform case, i.e., many configurations
of features are present, we should see that both importance measures
are equal. More formally, let $U=\{1,\ldots,k\}$, and let $\mathcal{D}$ be some dataset.
Suppose for each $i\in U$ we have some feature attribution function $\gamma_i:\mathcal{D}\to \mathbb{R}$.
Let us consider the \emph{global impact} of the feature over dataset $\mathcal{D}$
measured as
\begin{equation*}
  \Gamma_i=\sum_{x\in \mathcal{D}} |\gamma_i(x)|.
\end{equation*}
For example, we can set $\gamma_i=\phi_i$ to get a \emph{Shapley global impact} $\Phi_i$,
or we can set $\gamma_i=\beta_i$ to get a \emph{Banzhaf global impact} $B_i$.
This measure of global feature importance has also been used
by Lundberg et al.~\cite{Lundberg2020}.

In this section, we consider the set of \emph{monotone} functions on uniformly-distributed vertices of a $k$-dimensional hypercube and show that the Shapley global impacts and the Banzhaf global impacts always agree on them.
When defining Shapley and Banzhaf values we
use the idealized set function \linebreak
${g^*_x(S)=\EX[f(x_S,X_{\bar{S}})]}$ as in Section~\ref{s:prelims}.

More precisely, let the considered dataset be $\mathcal{D}=\{0,1\}^k$ and the data is spread uniformly,
i.e., for each a random variable $X\in\mathcal{D}$, we have $\mathbb{P}[X=x]=2^{-k}$ for all $x\in\mathcal{D}$.
For a subset of indices $S$ and a vector $x \in \{0,1\}^k$ let $\mathcal{D}_{S,x}$ be the set of all points in $\mathcal{D}$ which matches $x$ on the set $S$, i.e.,  $\mathcal{D}_{S,x} = \{x' | x'\in \mathcal{D}, x'_{S}=x_{S}\}$.

Let $f:\mathcal{D} \rightarrow \mathbb{R}$ be a function.
Let $x^{\neg i}$ denote~$x$ with the value of feature $i$ flipped to the opposite value.
We say that $f$ is monotone in feature $i$ if for each $\alpha\in \{0, 1\}$ all the numbers in the set
${\{f(x) - f(x^{\neg i}) | x\in \mathcal{D}, x_i=\alpha \}}$ have the same sign.

Let $w:2^U\to \mathbb{R}_{\geq 0}$ be a \emph{coefficient function} such that for any $i\in U$, we have
$\sum_{S\subseteq U\setminus\{i\}} w(S)=1$.
For data point $x\in\mathcal{D}$, we define the $w$-value of a feature $i$ as follows:
\begin{equation*}
  \omega_{i}^w(x) =\sum_{S\subseteq U \setminus \{i\}} w(S)\cdot \left(g^*_x(S\cup\{i\})-g^*_x(S)\right).
\end{equation*}
Observe that the appropriate choice of the coefficient function $w$ can yield Shapley values, Banzhaf values,
or, e.g., probabilistic values~\cite{weber1988probabilistic}.
Moreover, let the total $w$-impact be $\Omega_i^w=\sum_{x\in\mathcal{D}}\omega_i^w(x)$.

In our settings, we have:
\begin{equation*}
  g^*_x(S)=\EX[f(x_S,X_{\bar{S}})]=\sum_{x'\in \mathcal{D}_{S,x}}\frac{f(x')}{|\mathcal{D}_{S,x}|}
\end{equation*}

Now we prove that the impacts $\Omega_i^w$ for $i\in U$ are independent
of the coefficient function $w$ for our choice of $\mathcal{D}$, the uniform distribution,
and monotone functions $f:\mathcal{D}\to\mathbb{R}$.
This will prove that for all $i\in U$, $\Phi_i=B_i$.

We have:
\begin{align*}
  \omega_i^w(x) &= \sum_{S\subseteq U \setminus \{i\}} w(S)\cdot \left(g^*_x(S\cup\{i\})-g^*_x(S)\right)\\
  &= \sum_{S\subseteq U \setminus \{i\}} w(S)\left(\sum_{x'\in \mathcal{D}_{S\cup\{i\},x}}\frac{f(x')}{|\mathcal{D}_{S\cup\{i\},x}|}-\sum_{x'\in \mathcal{D}_{S,x}}\frac{f(x')}{|\mathcal{D}_{S,x}|}\right)\\
  &=\sum_{S\subseteq U \setminus \{i\}} w(S)
  \left(\sum_{x'\in \mathcal{D}_{S\cup\{i\},x}}\frac{f(x')}{2^{k-|S|-1}}-\sum_{x'\in \mathcal{D}_{S,x}}\frac{f(x')}{2^{k-|S|}}\right)\\
  &=\sum_{S\subseteq U \setminus \{i\}} w(S)\cdot 2^{|S|-k}
  \left(\sum_{x'\in \mathcal{D}_{S\cup\{i\},x}}2f(x')-\sum_{x'\in \mathcal{D}_{S,x}}f(x')\right)\\
  &=\sum_{S\subseteq U \setminus \{i\}} w(S)\cdot 2^{|S|-k}
  \left(\sum_{x'\in \mathcal{D}_{S\cup\{i\},x}}f(x')-f({x'}^{\neg i})\right).\\
\end{align*}

By the monotonicity of $f$, we get:
\begin{align*}
  |\omega_i^w(x)|&=\sum_{S\subseteq U \setminus \{i\}} w(S)\cdot 2^{|S|-k}
  \left(\sum_{x'\in \mathcal{D}_{S\cup\{i\},x}}|f(x')-f({x'}^{\neg i})|\right).\\
\end{align*}

Now we are ready to compute $\Omega_i^w$.
By changing the order of summation, we get:
\begin{align*}
  \Omega_i^w &= \sum_{x\in \mathcal{D}} |\omega_i^w(x)|\\
  &=\sum_{S\subseteq U \setminus \{i\}} w(S)\cdot 2^{|S|-k}
  \sum_{x\in \mathcal{D}}\left(\sum_{x'\in \mathcal{D}_{S\cup\{i\},x}}|f(x')-f({x'}^{\neg i})|\right)\\
  &=\sum_{S\subseteq U \setminus \{i\}} w(S)\cdot 2^{|S|-k}
  \sum_{x'\in \mathcal{D}}|f(x')-f({x'}^{\neg i})|\cdot |\{x:x'\in \mathcal{D}_{S\cup\{i\},x}\}|
\end{align*}

Now note that for a fixed $x'\in \mathcal{D}$, $x'$ is present in the set $\mathcal{D}_{S\cup\{i\},x}$
for $2^{k-|S|-1}$ distinct points $x\in \mathcal{D}$. Hence:
\begin{align*}
  \Omega_i^w &=\sum_{S\subseteq U \setminus \{i\}} w(S)\cdot 2^{|S|-k}
  \sum_{x'\in \mathcal{D}}|f(x')-f({x'}^{\neg i})|\cdot 2^{k-|S|-1}\\
  &=\frac{1}{2}\sum_{S\subseteq U \setminus \{i\}} w(S) \cdot \left(\sum_{x\in \mathcal{D}}|f(x)-f({x}^{\neg i})|\right)\\
  &=\frac{1}{2} \left(\sum_{x\in \mathcal{D}}|f(x)-f({x}^{\neg i})|\right) \cdot \left(\sum_{S\subseteq U \setminus \{i\}} w(S)\right) \\
  &=\frac{1}{2} \sum_{x\in \mathcal{D}}|f(x)-f({x}^{\neg i})|.
\end{align*}

The last equality follows since the coefficients $w(S)$ for $S\subseteq 2^U$ add up to $1$.
Thus $\Omega_{i}^w$ is independent of $w$.
It follows that $\Omega_i^w=\Phi_i=B_i$ for all $i\in U$, which completes the proof.

\section{Experimental Results}
\subsection{Datasets}
In our experiments, we used six datasets: four real and two artificial. These datasets were obtained by running either the sklearn implementation of Decision Trees (DT) or xgboost implementation of Gradient Boosting Decision Trees (GBDT) on three predictions datasets. These are some of the most popular algorithms for generating decision trees and are quite often used for large depths of the trees. Using large-depth trees is particularly beneficial for datasets with many features and complex relationship between features (see e.g., \cite{Bordag2021.04.18.21254782,Pham2019MultimodalDO} for a usage of trees of depth 100).  Let us emphasize that the large depth of the tree e.g. height 100 does not mean that the size of the tree is $2^{100}$, because trees might be (and usually are) unbalanced. To simplify the experiments and reduce the running times of experiments we trained the DT algorithm for only one tree.  The algorithm's parameters and the basic dataset's descriptions are as follows:
\begin{enumerate}[leftmargin=*]
  \item \bos{}~\cite{bostondataset}. This small prediction dataset contains information concerning housing in the
area of Boston Massachusetts. The task is to predict the price of the house. The dataset
    contains 506 rows and 13 columns. The decision tree (DT) was trained with tree\_depth equal to 10, all of the other parameters are set set to default values.
    The parameters used for training xgboost were: 100 iterations, max depth 6, and learning rate equal to 0.01.

 \item \nh{}. The same dataset that was used in previous work on model interpretability \cite{Lundberg2020}. The
dataset contains 8023 rows and 79 columns. The parameters used for training were the
same as in the original paper. The DT algorithm was trained with tree\_depth equal to 40, all of the other parameters are set to default values. The parameters used for training xgboost were: 250 iterations, max depth 4 and learning rate equal to 0.2.

 \item \hi{}~\cite{healthdataset}. A medium size dataset for predicting who might be
interested in health insurance. The dataset contains 304887 rows and 14 columns.
The DT algorithm was trained with tree\_depth equal to 60, all of the other parameters are set set to default values. The parameters used for training xgboost were: 250 iterations, max depth 4 and learning rate
equal to 0.2.

\item \fl{}~\cite{flightsdataset}. A large dataset for predicting the flights' delays. The
dataset contains 1543718 rows and 647 columns. The large number of columns was
caused by one-hot encoding 'UniqueCarrier', 'Origin', 'Dest',
    'CancellationCode' in a standard way, i.e., for each possible value $v$ of a given column $c$ we created additional categorical column $c\_v$ (with a value from $\{0,1\}$) indicating that the value of $c$ equals $v$ iff the value of $c\_v$ equals~$1$. The DT algorithm was trained with tree\_depth equal to 100, all of the other parameters are set set to default values.  The parameters used for training xgboost
    were: 250 iterations, max depth 10, and learning rate~0.2.
    
   \end{enumerate}

The algorithms were not extensively tuned since the main goal is interpreting
models not optimizing them.  The large trees are used mainly for comparing running times and elucidate numerical problems. We will refer to the above
datasets by adding ``DT'' and ``GBDT'' suffixes to the ordinal name of the prediction dataset.

We also prepared two synthetic instances with known exact answers (for both Shapley and Banzhaf values):
\begin{enumerate}[leftmargin=*]
  \item \synd{}. This instance contains one tree and one data point $x=[1,\dots,1]\in \mathbb{R}^d$. The tree consists of two subtrees of the same shape, both of them being full binary trees of depth $d-1$. All values in the leaves are equal to $0$ and $777$ in left and right subtree respectively. All leaves have coverages equal to $33$. In the internal nodes at depth $i$ of the tree we split on the condition $x_{d-i}< 1$ where
    the set of features is $U=\{1,\ldots,d\}$.

  \item \syns{}. This instance differs from the dense version only with the shape of the subtrees. Here each inner node of subtree has one leaf child and one non-leaf child.
\end{enumerate}

The only feature with a nonzero Shapley/Banzhaf value is the feature $d$ used to split at depth $0$. 
All the other features have zero Shapley/Banzhaf values by the \emph{Sensitivity} axiom.
\subsection{Algorithms}
In our experiments, we have tested four implementations of algorithms:
\begin{itemize}
  \item \shapours{} -- our implementation of the $O(TLD^2)$-time algorithm computing Shapley values based on~\cite{Lundberg2020},
  \item \shapfast{} -- an implementation of our asymptotically faster $O(TLD)$-time Shapley values algorithm,
  \item \shaporig{} -- implementation of the \texttt{TREESHAP\_PATH} algorithm~\cite{Lundberg2020} from the SHAP package~\cite{shappackage},
  \item \ban{} -- an implementation of our $O(TL)$-time algorithm computing Banzhaf values for tree ensemble models.
\end{itemize}
The implementations \shapours{} and \shapfast{} are consistent,
i.e., produce the same results.
There are small differences between
\shaporig{} and our implementations of Shapley values
but the mean average percentage difference between values is less than $1\%$. We suspect
that these differences are due to numerical differences in implementations.

\begin{figure}[h!bt]
 \centering
\includegraphics[width=0.5\linewidth,height=5cm]{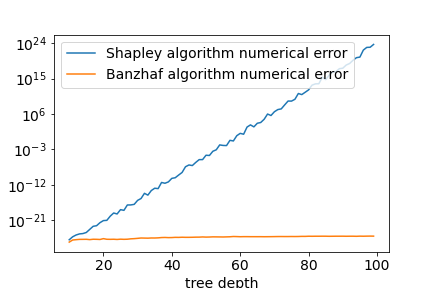}
  \caption{The numerical error for \texttt{SYNTHETIC\_SPARSE}. One can see that if the tree depth is more than apx. 50, the error makes the current implementation of Shapley values unusable. Moreover, the error grows exponentially in the depth of the tree.}
\label{fig:numerical}
\end{figure}

\subsection{Numerical Accuracy}
\label{section:numerical}
We observe that the algorithms for computing Shapley values -- both our version (in both variants with complexities
$O(TLD^2)$ and $O(TLD)$ resp.) and the original one (as implemented in the SHAP package~\cite{shappackage}) -- suffer numerical problems for large trees. To show that, we consider a simple artificial tree for which we know the answer for both Shapley values and Banzhaf values, namely the \texttt{SYNTHETIC\_SPARSE} instance.
We have observed that for trees of depth apx. 50 errors dominate completely the results, i.e., errors become apx.~1.
In Figure \ref{fig:numerical} we visualise the numerical errors
for Shapley values and Banzhaf values in the case of \texttt{SYNTHETIC\_SPARSE} instance.

For more realistic settings we observe that numerical errors that can alter results even for smaller trees used in practice, e.g., there are differences in the ordering of the first 10 features for different Shapley-based implementations on the \nhgb{} dataset. Hence, usage of Shapley values in practice requires additional care to control these errors, as otherwise one
is not guaranteed to obtain correct results.

\subsection{Comparison of Global Impacts}
Now we present a comparison between \ban{} and \shaporig{}.
In Figures~\ref{fig:boston},~\ref{fig:health},~\ref{fig:nhanes},\ref{fig:flights}~\ref{fig:bostondt},~\ref{fig:healthdt},~\ref{fig:nhanesdt},~and~\ref{fig:flightsdt} we show the global impacts of features, as defined in the previous section.
We confirm that the differences in this case are very small for small trees and visibly larger for large trees.
However, one can observe differences for specific data points. In Table~\ref{tab:cayley}, we present the average Cayley distance between the feature orderings
derived from the computed Shapley and Banzhaf values. In particular, when many features are used by the model, the difference in orderings becomes larger, as 
individual values become smaller. 

\begin{table}[h!]
  {
  \begin{center}
  \scalebox{0.85}{
\begin{tabular}{ |c|c|c|c|c| }
 \hline
 Ins/n & 3 & 10 & 20 \\
\hline
BOS\_GBDT &  0.02 &  1.05 &    \\
\hline
NH\_GBDT & 0.01& 0.34  & 1.53     \\
\hline
HI\_GBDT & 0.02& 0.73  &      \\
\hline
FL\_GBDT & 0.23  & 3.08 &  8.63    \\
\hline
BOS\_DT &  0.08 &  1.73 &    \\
\hline
NH\_DT & 0.33 & 4.33  & 11.62   \\
\hline
HI\_DT & 1.10 & 6.79  &      \\
\hline
FL\_DT & 0.84  & 6.39 &   14.11\\
\hline
\end{tabular}}
\vspace{0.1cm}
  \caption{The average modified Cayley distance for the $n$ most important features for $n\in \{3,10,20\}$ between Shapley and Banzhaf values. The Cayley distance measures the number of swaps needed to switch from one permutation to another. In our modified version, we also support the case where the sets of considered most important features in the respective permutations are different. For a missing feature, we add it at the end of the permutation.}\label{tab:cayley}
\end{center}
  }
\end{table}

\begin{figure}[h!t]
 \centering
\begin{subfigure}[t]{.23\textwidth}
\includegraphics[width=\linewidth,height=3.5cm]{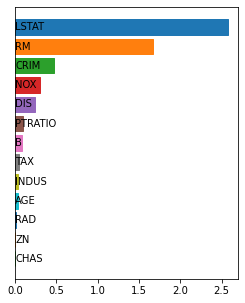}
  \caption{Original Shapley values.}
\end{subfigure}
\begin{subfigure}[t]{.23\textwidth}
  \includegraphics[width=\linewidth,height=3.5cm]{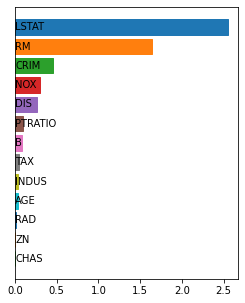}
  \caption{Banzhaf values.}
\end{subfigure}
  \caption{The global impacts of the individual features for the \bosgb{} dataset. We observe that the plots are indistinguishable.}
\label{fig:boston}
\end{figure}
\begin{figure}[h!t]
 \centering
\begin{subfigure}[t]{.23\textwidth}
\includegraphics[width=\linewidth,height=4cm]{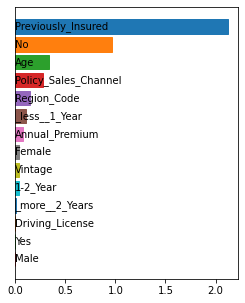}
  \caption{Original Shapley values.}
\end{subfigure}
\begin{subfigure}[t]{.23\textwidth}
\includegraphics[width=\linewidth,height=4cm]{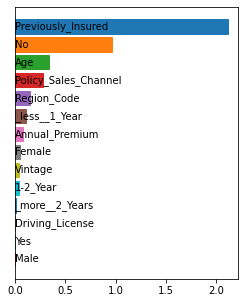}
  \caption{Banzhaf values.}
\end{subfigure}
\caption{The features' global impacts for the \higb{} dataset. We observe that the plots are indistinguishable.}
\label{fig:health}
\end{figure}
\begin{figure}[h!t]
 \centering
\begin{subfigure}[t]{.23\textwidth}
\includegraphics[width=\linewidth,height=4cm]{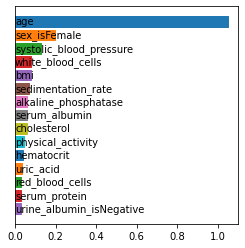}
  \caption{Original Shapley value.}
\end{subfigure}
\begin{subfigure}[t]{.23\textwidth}
\includegraphics[width=\linewidth,height=4cm]{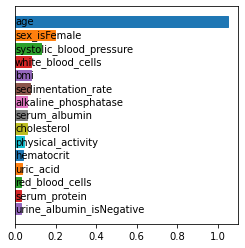}
  \caption{Banzhaf values.}
\end{subfigure}
  \caption{The global impacts of the individual features for the \nhgb{} dataset. We observe that the plots are indistinguishable.}
\label{fig:nhanes}
\end{figure}
\begin{figure}[!ht]
 \centering
\begin{subfigure}[t]{.23\textwidth}
\includegraphics[width=\linewidth,height=4.2cm]{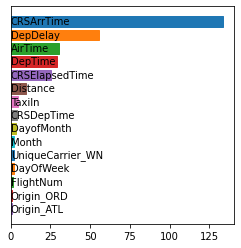}
  \caption{Original Shapley value.}
\end{subfigure}
\begin{subfigure}[t]{.23\textwidth}
\includegraphics[width=\linewidth,height=4.2cm]{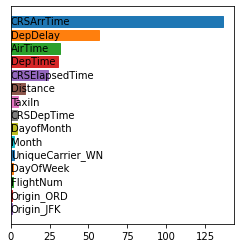}
  \caption{Banzhaf values.}
\end{subfigure}
\caption{The global impacts of the individual features for the \flgb{} dataset. We observe small differences in the ordering.}
\label{fig:flights}
\end{figure}

\begin{figure}[h!t]
 \centering
\begin{subfigure}[t]{.23\textwidth}
\includegraphics[width=\linewidth,height=3.5cm]{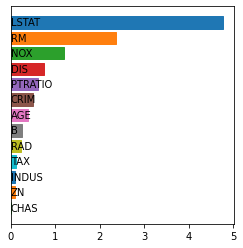}
  \caption{Original Shapley values.}
\end{subfigure}
\begin{subfigure}[t]{.23\textwidth}
\includegraphics[width=\linewidth,height=3.5cm]{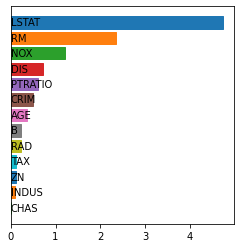}
  \caption{Banzhaf values.}
\end{subfigure}
  \caption{The global impacts of the individual features for the \bosdt{} dataset. We observe minor differences between plots.}
\label{fig:bostondt}
\end{figure}
\begin{figure}[h!t]
 \centering
\begin{subfigure}[t]{.23\textwidth}
\includegraphics[width=\linewidth,height=4cm]{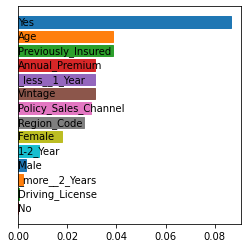}
  \caption{Original Shapley values.}
\end{subfigure}
\begin{subfigure}[t]{.23\textwidth}
\includegraphics[width=\linewidth,height=4cm]{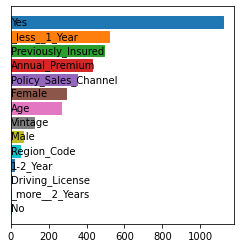}
  \caption{Banzhaf values.}
\end{subfigure}
\caption{The features' global impacts for the \hidt{} dataset. We observe . We observe major differences between plots..}
\label{fig:healthdt}
\end{figure}
\begin{figure}[h!t]
 \centering
\begin{subfigure}[t]{.23\textwidth}
\includegraphics[width=\linewidth,height=4cm]{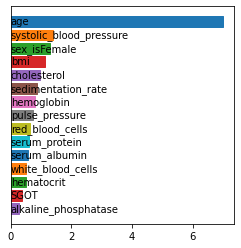}
  \caption{Original Shapley value.}
\end{subfigure}
\begin{subfigure}[t]{.23\textwidth}
\includegraphics[width=\linewidth,height=4cm]{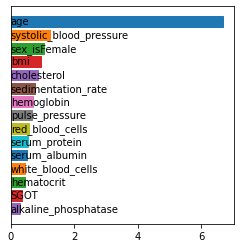}
  \caption{Banzhaf values.}
\end{subfigure}
  \caption{The global impacts of the individual features for the \nhdt{} dataset. We observe that the plots are indistinguishable.}
\label{fig:nhanesdt}
\end{figure}
\begin{figure}[!ht]
 \centering
\begin{subfigure}[t]{.23\textwidth}
\includegraphics[width=\linewidth,height=4.2cm]{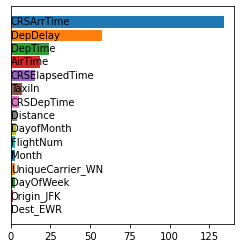}
  \caption{Original Shapley value.}
\end{subfigure}
\begin{subfigure}[t]{.23\textwidth}
  \includegraphics[width=\linewidth,height=4.2cm]{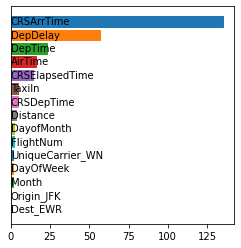}
  \caption{Banzhaf values.}
\end{subfigure}
\caption{The global impacts of the individual features for the \fldt{} dataset. We observe minor differences in the ordering.}
\label{fig:flightsdt}
\end{figure}

\subsection{Feature Values' Comparision}\label{s:detailed}
In order to pinpoint the reason why explanations for these models differ we have looked at
specific instances. In particular, we present a detailed comparison of feature importances
for individual predictions produced by the implemented algorithms. 
Figures~\ref{fig:bostonl},~\ref{fig:nhanesl},~\ref{fig:hil},~\ref{fig:flightsl}
show the MAE (Mean Average Error) and RMSE (Root Mean Square Error) between the respective Shapley and Banzhaf
values for each individual feature through all the data points in the data sets
\bosgb{}, \nhgb{},~\higb{},~and~\linebreak \flgb{}.
Formally, for each dataset $\mathcal{D}$ out of those and each feature $i$ used there, these are defined as:

\[\text{MAE}_i= \frac{1}{|\mathcal{D}|}\sum_{x\in\mathcal{D}}|\phi_i(x) - \beta_i(x)| \qquad
\text{RMSE}_i = \sqrt{\frac{1}{|\mathcal{D}|}\sum_{x\in\mathcal{D}}(\phi_i(x)^2 - \beta_i(x))^2}\]

Above, $\phi_i(x)$ denotes the Shapley explanation of $f(x)$ for data point $x\in\mathcal{D}$, as computed by \shaporig{}.
Similarly, $\beta_i(x)$ denotes the Banzhaf explanation of $f(x)$ for $x\in\mathcal{D}$ as computed by \ban{}.

We observe that errors in 
Figures~\ref{fig:bostonl},~\ref{fig:nhanesl},~\ref{fig:hil},~\ref{fig:flightsl}
lie in the range of few small percent in relation to the possible ranges of global impacts. These errors indicate that when looking at specific data points
one should expect only some differences in ordering of the features and only for features with similar scores. 
However, this is the case only for the GBDT datasets, whereas when the DT algorithm is used these errors become dominating. 
As these plots do not seem to contain any valuable information, we have decided not to include them. 
This is further illustrated in the next figures which present scores for specific data-points.

\begin{figure}[!h]
 \centering
\begin{subfigure}[t]{.41\textwidth}
\includegraphics[width=\linewidth]{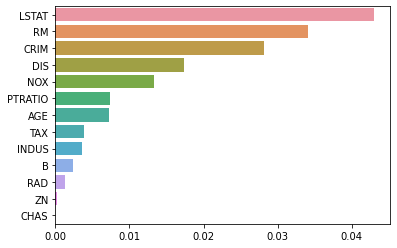}
  \caption{MAE error.}
\end{subfigure}
\hspace*{2mm}
\begin{subfigure}[t]{.41\textwidth}
\includegraphics[width=\linewidth]{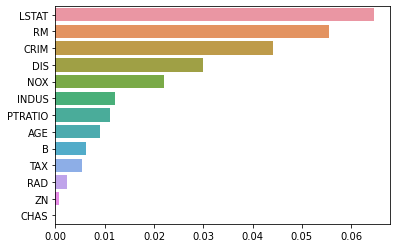}
  \caption{RMSE error.}
\end{subfigure}
\caption{The MAE and RMSE error between Banzhaf and Shapley values for the \bosgb{} dataset.}
\label{fig:bostonl}
\end{figure}
\begin{figure}[!h]
 \centering
\begin{subfigure}[t]{.48\textwidth}
\includegraphics[width=\linewidth]{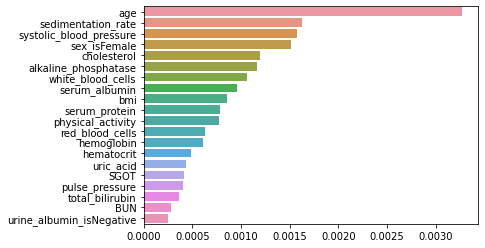}
  \caption{MAE error.}
\end{subfigure}
\begin{subfigure}[t]{.48\textwidth}
\includegraphics[width=\linewidth]{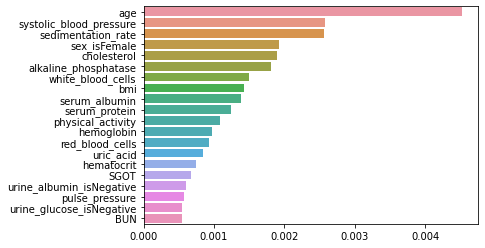}
  \caption{RMSE error.}
\end{subfigure}
\caption{The MAE and RMSE error between Banzhaf and Shapley values for the \nhgb{} dataset.}
\label{fig:nhanesl}
\end{figure}
\begin{figure}[!h]
 \centering
\begin{subfigure}[t]{.48\textwidth}
\includegraphics[width=\linewidth]{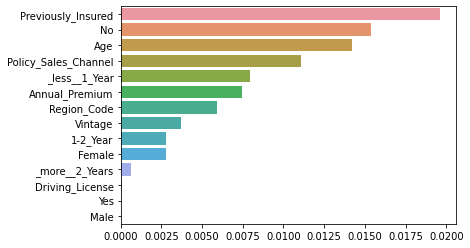}
  \caption{MAE error.}
\end{subfigure}
\begin{subfigure}[t]{.48\textwidth}
\includegraphics[width=\linewidth]{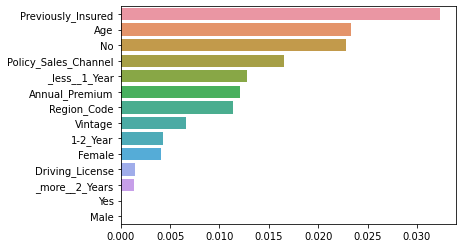}
  \caption{RMSE error.}
\end{subfigure}
\caption{The MAE and RMSE error between Banzhaf and Shapley values for the \higb{} dataset.}
\label{fig:hil}
\end{figure}
\begin{figure}[!h]
 \centering
\begin{subfigure}[t]{.48\textwidth}
\includegraphics[width=\linewidth]{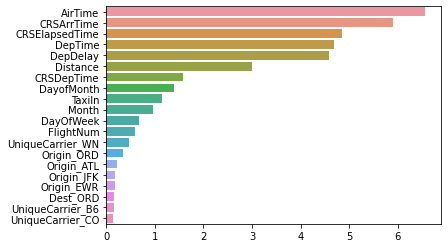}
  \caption{MAE error.}
\end{subfigure}
\begin{subfigure}[t]{.48\textwidth}
\includegraphics[width=\linewidth]{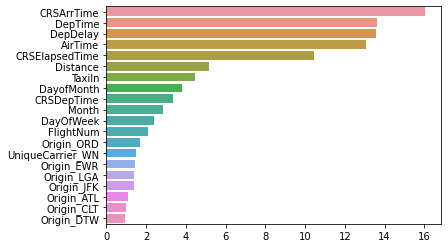}
  \caption{RMSE error.}
\end{subfigure}
\caption{The MAE and RMSE error between Banzhaf and Shapley values for the \flgb{} dataset.}
\label{fig:flightsl}
\end{figure}

In Figures \ref{fig:bostongb_bad},~\ref{fig:nhanesgb_bad}
,~\ref{fig:higb_bad},~\ref{fig:flightsgb_bad},~\ref{fig:boston_bad},~\ref{fig:nhanes_bad}
,~\ref{fig:hi_bad},~\ref{fig:flights_bad} we present ``bad'' examples for which the order of important features differs the most for Banzhaf and Shapley values for large trees.
For small trees, differences do not change the overall interpretations, i.e., the ordering
of the top features remains essentially the same, whereas for large trees we observe major differences in feature importance.
In Figure \ref{fig:hi_bad}, one can see larger differences in values, e.g., there is a difference in the most important value.
These results seem to seemingly falsify Hypothesis~\ref{hyp-1}. However, as we observed in Section~\ref{section:numerical} that
numerical errors can be dominating in Shapley values computations. This is visualized on 
Figure \ref{fig:hi_bad_shap} that visualizes differences between different implementations of Shapley values. 
Hence, the numerical errors can make specific Shapley explanations unusable for DT algorithms. 

\begin{figure}[!h]
 \centering
\begin{subfigure}[t]{.23\textwidth}
\includegraphics[width=\linewidth]{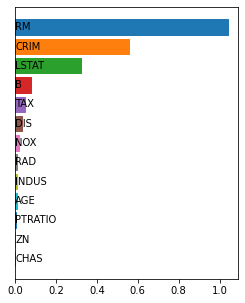}
  \caption{Original shapley value.}
\end{subfigure}
\begin{subfigure}[t]{.23\textwidth}
\includegraphics[width=\linewidth]{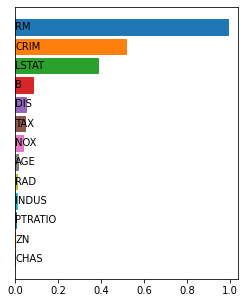}
  \caption{Banzhaf values.}
\end{subfigure}
\caption{A ``bad'' example: Banzhaf and Shapley values for a data point for which they differ the most on the \bosgb{} dataset.}
\label{fig:bostongb_bad}
\end{figure}

\begin{figure}[!h]
 \centering
\begin{subfigure}[t]{.23\textwidth}
\includegraphics[width=\linewidth]{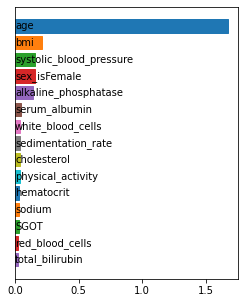}
  \caption{Original shapley value.}
\end{subfigure}
\begin{subfigure}[t]{.23\textwidth}
\includegraphics[width=\linewidth]{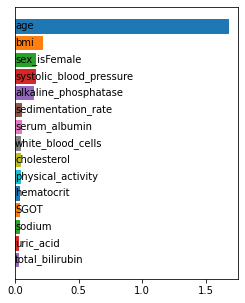}
  \caption{Banzhaf values.}
\end{subfigure}
\caption{A ``bad'' example: Banzhaf and Shapley values for a data point for which they differ the most on the \nhgb{} dataset.}
\label{fig:nhanesgb_bad}
\end{figure}

\begin{figure}[!h]
 \centering
\begin{subfigure}[t]{.23\textwidth}
\includegraphics[width=\linewidth]{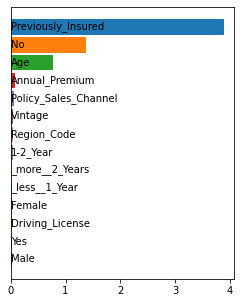}
  \caption{Original shapley value.}
\end{subfigure}
\begin{subfigure}[t]{.23\textwidth}
\includegraphics[width=\linewidth]{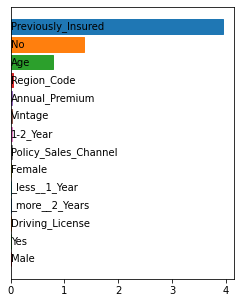}
  \caption{Banzhaf values.}
\end{subfigure}
\caption{A ``bad'' example: Banzhaf and Shapley values for a data point for which they differ the most on the \higb{} dataset.}
\label{fig:higb_bad}
\end{figure}

\begin{figure}[!h]
 \centering
\begin{subfigure}[t]{.23\textwidth}
\includegraphics[width=\linewidth]{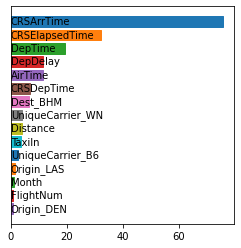}
  \caption{Original shapley value.}
\end{subfigure}
\begin{subfigure}[t]{.23\textwidth}
\includegraphics[width=\linewidth]{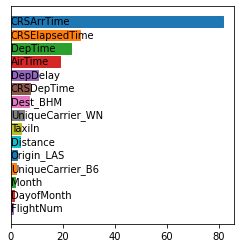}
  \caption{Banzhaf values.}
\end{subfigure}
\caption{A ``bad'' example: Banzhaf and Shapley values for a data point for which they differ the most on the \flgb{} dataset.}
\label{fig:flightsgb_bad}
\end{figure}

\begin{figure}[!h]
 \centering
\begin{subfigure}[t]{.23\textwidth}
\includegraphics[width=\linewidth]{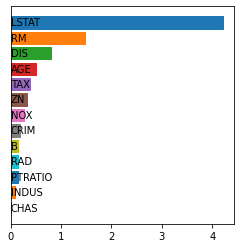}
  \caption{Original shapley value.}
\end{subfigure}
\begin{subfigure}[t]{.23\textwidth}
\includegraphics[width=\linewidth]{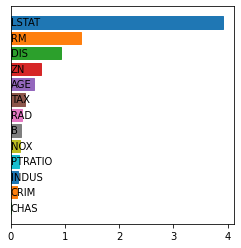}
  \caption{Banzhaf values.}
\end{subfigure}
\caption{A ``bad'' example: Banzhaf and Shapley values for a data point for which they differ the most on the \bosdt{} dataset.}
\label{fig:boston_bad}
\end{figure}

\begin{figure}[!h]
 \centering
\begin{subfigure}[t]{.23\textwidth}
\includegraphics[width=\linewidth]{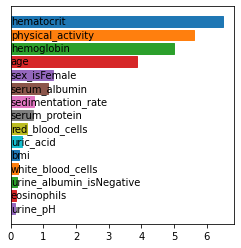}
  \caption{Original shapley value.}
\end{subfigure}
\begin{subfigure}[t]{.23\textwidth}
\includegraphics[width=\linewidth]{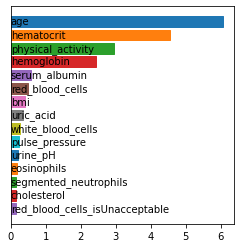}
  \caption{Banzhaf values.}
\end{subfigure}
\caption{A ``bad'' example: Banzhaf and Shapley values for a data point for which they differ the most on the \nhdt{} dataset. We observe large differences caused by numerical errors.}
\label{fig:nhanes_bad}
\end{figure}

\begin{figure}[!h]
 \centering
\begin{subfigure}[t]{.23\textwidth}
\includegraphics[width=\linewidth]{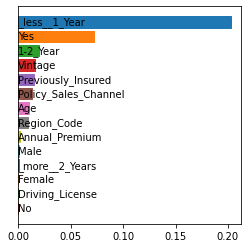}
  \caption{Original shapley value.}
\end{subfigure}
\begin{subfigure}[t]{.23\textwidth}
\includegraphics[width=\linewidth]{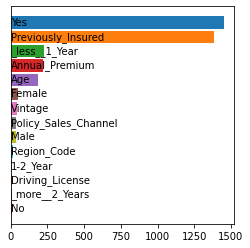}
  \caption{Banzhaf values.}
\end{subfigure}
\caption{A ``bad'' example: Banzhaf and Shapley values for a data point for which they differ the most on the \hidt{} dataset. We observe large differences caused by numerical errors.}
\label{fig:hi_bad}
\end{figure}

\begin{figure}[!h]
 \centering
\begin{subfigure}[t]{.23\textwidth}
\includegraphics[width=\linewidth]{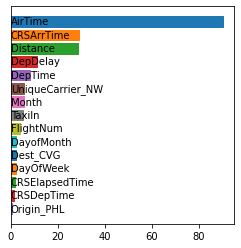}
  \caption{Original shapley value.}
\end{subfigure}
\begin{subfigure}[t]{.23\textwidth}
\includegraphics[width=\linewidth]{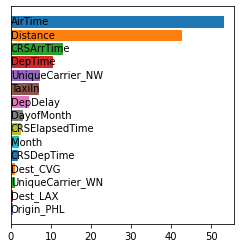}
  \caption{Banzhaf values.}
\end{subfigure}
\caption{A ``bad'' example: Banzhaf and Shapley values for a data point for which they differ the most on the \fldt{} dataset. We observe large differences caused by numerical errors.}
\label{fig:flights_bad}
\end{figure}

\begin{figure}[!h]
 \centering
\begin{subfigure}[t]{.21\textwidth}
\includegraphics[width=\linewidth]{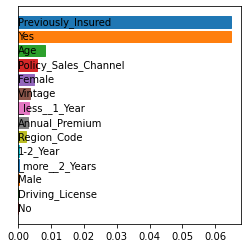}
  \caption{Original Shapley values.}
\end{subfigure}
\begin{subfigure}[t]{.21\textwidth}
\includegraphics[width=\linewidth]{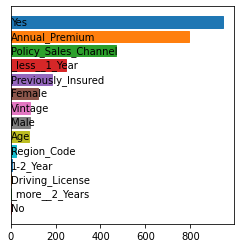}
  \caption{Original Shapley values alternative implementation.}
\end{subfigure}
\caption{A ``bad'' example: comparing two implementations of Shapley values: \shaporig{} and \shapours{} for a data point for which they differ the most on the \hidt{} dataset. We observe large differences caused by numerical errors.}
\label{fig:hi_bad_shap}
\end{figure}

\subsection{Comparison of Running Times}
In this section, we compare the running times of the algorithms. In Table~\ref{tab:running_times} we show the running times for different examples. It can be seen
that \ban{} is faster than all other methods, and using it can lead to considerable time savings for larger data-sets. For small depths,
as used in these examples, \shapours{} and \shapfast{} essentially run in the same time, whereas \shaporig{} is faster due to different implementation.

However, this changes with the increased depth of the trees as shown in Figure \ref{fig:running_art}. We included our implementation of the original algorithm of Lundberg et al. (\shapours{}) to avoid unfair comparison, e.g., differences in efficiency of data structures and preprocessing. We observe that the asymptotically faster (by a factor of $D$) versions of the algorithms for computing Shapley values and Banzhaf values are orders of magnitudes faster than the previous versions.
\begin{table}[h!]
\begin{center}
    \scalebox{0.85}{
\begin{tabular}{ |c|c|c|c|c| }
 \hline
  Ins &\ban{}&\shapfast{}&\shaporig{}&\shapours{} \\
 \hline
 BS\_GBDT &  0.48 s &  0.56 s &  0.63 s &  0.70 s \\
\hline
HI\_GBDT &  23.63 s &  51.73 s &  1 m 9 s &  35.32 s \\
\hline
NH\_GBDT&  50.20 s &  2 m 28 s &  2 m 56 s &  1 m 28 s \\  
\hline
FL\_GBDT &  13 m 18 s &  1 h 47 m &  1 h 50 m &  48 m 8 s \\
\hline
BS\_DT &  0.41 s &  0.42 s  &  0.41 s&  0.42 s \\
\hline
 NH\_DT&  3.57 s&  34.92 s  &  42.87 s&  45.58 s  \\
\hline
HI\_DT &  4 m 55 s&  23 m 18 s&  30 m 55 s &  35 m 3 s   \\
\hline
FL\_DT &  14 m 28 s  &  5 h 23 m &  5 h 9 m&  5 h 40 m \\
\hline
\end{tabular}}
\vspace{0.1cm}
 \caption{The comparison of running times for different algorithms. We observe that Banzhaf values implementation (\texttt{ban}) is substantially faster than all the Shapley value implementations on each instance. }
  \label{tab:running_times}
\end{center}
\end{table}

\vspace{0mm}
\begin{figure}[!h]
 \centering
\includegraphics[width=0.55\linewidth,height=6cm]{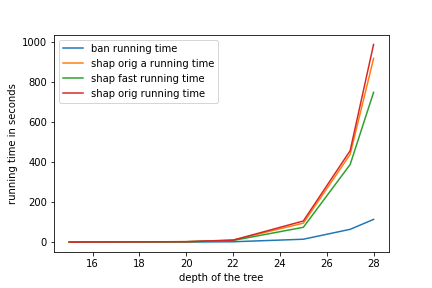}
  \caption{The comparison of running times for different tree sizes for the \texttt{SYNTHETIC\_DENSE} instance.}
\label{fig:running_art}
\end{figure}

\subsection{Experimental Setup}
All of the experiments were done using Intel(R) Xeon(R) CPU E5-2630 v4 @ 2.20GHz with 512 Gb of RAM using only one thread for computation. The operating system was Ubuntu 18.04.2 LTS.
The binaries were compiled using clang version 6.0.0-1ubuntu2 with -O3 optimization.

\section{The Basic Algorithm for Shapley and Banzhaf Values}\label{a:simple}
In this section we describe in detail a variant of the ${O(TLD^2+n)}$ time
\texttt{TREESHAP\_PATH} algorithm
of Lundberg et al.~\cite{Lundberg2020}. We also
show how it is adjusted to compute the Banzhaf values instead.
Recall that
$T$ denotes the number of trees in the ensemble, and $L$ and $D$ bound the numbers
of leaves and the depth of each of these trees, respectively.
We first consider the case $T=1$ (i.e., that the ensemble consists of
a single tree $\tr$) and then (in Section~\ref{s:multiple}) we show how
to handle larger tree ensembles.

For some non-root node $v$ of the tree,
denote by $z_v$ the feature in the parent node $p_v$ of $v$,
i.e., $z_v=d_{p_v}$.
Recall that $F_v$ denotes the set of features in the
ancestors of $v$, excluding the feature in $v$.
So we have $F_\rho=\emptyset$ and $F_v=F_{p_v}\cup \{z_v\}.$

Recall that our ultimate goal is to compute for each $i\in U$:
\begin{enumerate}
  \item Shapley values defined as:
    \begin{equation*}
      \phi_i=\frac{1}{n}\sum_{S\subseteq U\setminus\{i\}} {\binom{n-1}{|S|}}^{-1}\left(g(S\cup\{i\})-g(S)\right).
    \end{equation*}
  \item Banzhaf values defined as:
    \begin{equation*}
      \beta_i=\frac{1}{2^{n-1}}\sum_{S\subseteq U\setminus\{i\}}\left(g(S\cup \{i\})-g(S)\right).
    \end{equation*}
\end{enumerate}
where $g$ is defined to be the output of Algorithm~\ref{alg:est}. We can also write:
\begin{equation}\label{eq:gdef}
  g(S):=\sum_{l\in\lvs(\tr)} P[l,S]\cdot f(l),
\end{equation}
where the values $P[\cdot,\cdot]$ are defined using Algorithm~\ref{alg:est}.
More precisely, if Algorithm~\ref{alg:est} is executed with subset $S\subseteq U$,
then $P[v,S]$ for each $v\in\tr$ equals the weight
that the ancestor recursive calls
assign to the subtree rooted at $v$.
Formally, $P[\rho,S]=1$, and for any $v\neq \rho$,
\begin{equation*}
  P[v,S]=\begin{cases}
      P[p_v,S]\cdot \frac{r_v}{r_{p_v}} & \text{ if } z_v\notin S,\\
      P[p_v,S]\cdot [x_{z_v}< t_{p_v}] & \text{ if }  z_v\in S, v=a_{p_v},\\
      P[p_v,S]\cdot [x_{z_v}\geq t_{p_v}] & \text{ if }  z_v\in S, v=b_{p_v}.\\
  \end{cases}
\end{equation*}

Denote by $I_{v,y}$ the minimal interval that $x_{y}$ has to belong to in order to $x$ end up being evaluated to a leaf in the subtree of~$v$,
i.e., so that $f(x)\in \{f(l):l\in \lvs(\tr[v])\}$. In particular, if no ancestor of $v$ stores the feature $y$, we set $I_{v,y}=(-\infty,\infty)$.
We say that $I_{v,y}$ is \emph{non-trivial} if
$I_{v,y}\neq I_{p_v,y}$.
Observe that $I_{v,y}$ can only be non-trivial if
$y=z_v$. Moreover, in that case we have
\begin{equation*}I_{v,y}=\begin{cases}
  I_{p_v,y}\cap (-\infty,t_{p_v}] & v=a_{p_v},\\
  I_{p_v,y}\cap (t_{p_v},\infty) & v=b_{p_v}.
\end{cases}
\end{equation*}
Hence, we obtain the following.
\begin{observation}
  All the non-trivial intervals $I_{v,y}$ for all possible $(v,y)\in \tr\times U$, can be computed in $O(L)$ time.
\end{observation}

Let $y\in U$ be some feature.
Denote by $c_v(y)$ the product of all the values $r_u/r_{p_u}$ such that $u$ is non-root (weak) ancestor of $v$ and $z_u=y$. In particular, if $y\notin F_v$, then $c_v(y)=1$.
Similarly as was the case for the intervals
$I_{v,y}$, a value $c_v(y)$ can only differ
from $c_{p_v}(y)$ if $y=z_v$ (then we have
$c_v(y)=c_{p_v}(y)\cdot r_v/r_{p_v}$). Hence,
all such non-trivial values $c_v(y)$ can be computed
in $O(L)$ time as well.

Let us start by describing the basic algorithm
for Shapley values.
To proceed, we introduce the following crucial notation.
For any set $G\subseteq U$, let
\begin{equation}\label{eq:dpdef}
  \phi(v,G,k):=\frac{1}{|G|+1}\sum_{\substack{S\subseteq G\\|S|=k}} \binom{|G|}{k}^{-1}\cdot P[v,S].
\end{equation}
Roughly speaking, the values $\phi(v,G,k)$ can be computed using dynamic programming. Let us put
\begin{equation*}
  \phi(v,G):=\sum_{k=0}^{|G|}\phi(v,G,k).
\end{equation*}
\subsection{Leaf Contributions}

In this section we show that computing $\phi(l,F_l\setminus\{i\})$
for all pairs ${l\in \lvs(\tr)}$, $i\in F_l$ is sufficient
to get all the sought Shapley values.
In order to prove that, let us first observe some
simple properties of the values $P[\cdot,\cdot]$.

\begin{lemma}\label{c:add_feature}
Let $v\in \tr$ and $Q\subseteq U$ and $y\in U\setminus Q$. Then:
\begin{equation*}
  P[v,Q\cup\{y\}]=P[v,Q]\cdot \frac{[x_y\in I_{v,y}]}{c_v(y)}.
\end{equation*}
\end{lemma}
\begin{proof}
  This can be proven by induction on the depth of $v$ in~$\tr$.
  The claim holds obviously for $v=\rho$.
  So suppose $v$ is non-root.
  First recall that if $z_v\neq y$, we have
  $[x_y\in I_{v,y}]={[x_y\in I_{p_v,y}]}$ and $c_v(y)=c_{p_v}(y)$.

  Suppose $z_v\notin Q\cup\{y\}$. Then:
  \begin{align*}
    P[v,Q\cup\{y\}]&=P[p_v,Q\cup\{y\}]\cdot \frac{r_v}{r_{p_v}}\\
                   &=P[p_v,Q]\cdot\frac{[x_y\in I_{v,y}]}{c_v(y)}\cdot
    \frac{r_v}{r_{p_v}}\\
    &=P[v,Q]\cdot\frac{[x_y\in I_{v,y}]}{c_v(y)}.
  \end{align*}
  Otherwise, $z_v\in Q\cup \{y\}$.
  Assume wlog. $v=a_{p_v}$ -- the case $v=b_{p_v}$ is symmetric.
  We have:
  \begin{align*}
    P[v,Q\cup\{y\}]&=P[p_v,Q\cup\{y\}]\cdot [x_{z_v}< t_{z_v}]\\
                   &=P[p_v,Q]\cdot \frac{[x_y\in I_{{p_v},y}]}{c_{p_v}(y)}\cdot [x_{z_v}< t_{z_v}]
  \end{align*}
  If $z_v=y$, then $z_v\notin Q$ and we have
  $c_v(y)=c_{p_v}(y)\cdot \frac{r_v}{r_{p_v}}$ and
  $[x_y\in I_{v,y}]=[x_y\in I_{p_v,y}]\cdot [x_y< t_{z_v}]$.
  So in that
  case
  \begin{align*}
    P[v,Q\cup\{y\}]&=P[p_v,Q]\cdot \frac{[x_y\in I_{v,y}]}{c_v(y)}\cdot\frac{r_{p_v}}{r_v}\\
    &= P[v,Q]\cdot \frac{[x_y\in I_{v,y}]}{c_v(y)}
  \end{align*}
  If on the other hand we have $y\neq z_v\in Q$, then:
  \begin{align*}
    P[v,Q\cup\{y\}]=P[p_v,Q]\cdot \frac{[x_y\in I_{v,y}]}{c_{v}(y)}\cdot [x_{z_v}< t_{z_v}]
    =P[v,Q]\cdot \frac{[x_y\in I_{v,y}]}{c_{v}(y)}.
  \end{align*}
\end{proof}

\begin{lemma}\label{c:move_down}
Let $v\in \tr$ be non-root and $Q\subseteq U\setminus \{z_v\}$. Then
\begin{equation*}
  P[v,Q]=P[p_v,Q]\cdot \frac{c_v(z_v)}{c_{p_v}(z_v)}.
\end{equation*}
\end{lemma}
\begin{proof}
  It is enough to note that $\frac{c_v(z_v)}{c_{p_v}(z_v)}=\frac{r_v}{r_{p_v}}$.
\end{proof}

The following lemma, which we prove later on,
states an intuitive fact that $\phi(v,G)$ does
not depend on the features in $G$ that do not appear
in the ancestors of $v$.

\begin{lemma}\label{l:sum}
Let $v\in \tr$ and $G\subseteq U$. Suppose $y\in U\setminus G\setminus F_v$. Then:
  \begin{equation*}
  \phi(v,G\cup \{y\})=\phi(v,G).
\end{equation*}
\end{lemma}

Let us now expand the sum~(\ref{eq:shap}) using~(\ref{eq:gdef})
to pinpoint the
individual contributions of each leaf $l$ into $\phi_i$.
\begin{align*}
  \phi_i&=\frac{1}{n}\sum_{S\subseteq U\setminus\{i\}} {\binom{n-1}{|S|}}^{-1}\left(g(S\cup\{i\})-g(S)\right)\\
  &=\frac{1}{n}\sum_{S\subseteq U\setminus\{i\}} {\binom{n-1}{|S|}}^{-1}\left(\sum_{l\in\lvs(\tr)}f(l)\left(P[l,S\cup\{i\}]-P[l,S]\right)\right)\\
\end{align*}
By subsequently applying Lemma~\ref{c:add_feature},
changing the summation order, using~(\ref{eq:gdef}), and
finally applying Lemma~\ref{l:sum}, we get:
\begin{align*}
  \phi_i&=\frac{1}{n}\sum_{S\subseteq U\setminus\{i\}} {\binom{n-1}{|S|}}^{-1}\left(\sum_{l\in\lvs(\tr)}f(l)\cdot P[l,S]\left(\frac{[x_i\in I_{l,i}]}{c_l(i)}-1\right)\right)\\
  &=\sum_{l\in\lvs(\tr)}f(l)\cdot \left(\frac{[x_i\in I_{l,i}]}{c_l(i)}-1\right)\left(\frac{1}{n}\sum_{S\subseteq U\setminus\{i\}} {\binom{n-1}{|S|}}^{-1}P[l,S]\right)\\
  &=\sum_{l\in\lvs(\tr)}f(l)\cdot \left(\frac{[x_i\in I_{l,i}]}{c_l(i)}-1\right)\cdot \phi(l,U\setminus\{i\})\\
\end{align*}
Since $\left(\frac{[x_i\in I_{l,i}]}{c_l(i)}-1\right)=0$ when $i\notin F_l$,
we actually have:
\begin{equation}\label{eq:leaf-contrib}
  \phi_i=\sum_{\substack{l\in\lvs(\tr)\\i\in F_l}}f(l)\cdot \left(\frac{[x_i\in I_{l,i}]}{c_l(i)}-1\right)\cdot \phi(l,F_l\setminus\{i\}).
\end{equation}

So indeed, Equation~(\ref{eq:leaf-contrib}) provides an $O(LD)$-time
reduction of
the problem of computing
Shapley values to computing the values $\phi(l,F_l\setminus\{i\})$
for all $l\in \lvs(\tr)$ and $i\in F_l$.

\subsection{Dynamic Programming}
In this section we show how the values $\phi(v,G,k)$ can be
computed recursively to avoid iterating through all subsets
in the sum~(\ref{eq:dpdef}).
For convenience let us define $\phi(v,G,k)=0$ for $k<0$ or
$k>|G|$.

\begin{lemma}\label{l:dp}
Let $v\in \tr$, $G\subseteq U$ and $k\in\{0,\ldots,|G|\}$. Let $y\in U\setminus G$. Then:
  \begin{align*}
    \phi(v,G\cup\{y\},k)&=\frac{|G|+1-k}{|G|+2}\cdot\phi(v,G,k)+\frac{k}{|G|+2}\cdot \frac{[x_y\in I_{v,y}]}{c_v(y)}\cdot \phi(v,G,k-1).
  \end{align*}
\end{lemma}
\begin{proof}
  Let $m=|G|+1$ and $X=\phi(v,G\cup\{y\},k)$. By Lemma~\ref{c:add_feature} we get:
  \begin{align*}X&=\sum_{\substack{S\subseteq G\cup\{y\}\\|S|=k}}\frac{1}{m+1} \binom{m}{k}^{-1} P[v,S]\\
    &=\sum_{\substack{S\subseteq G\\|S|=k}} \frac{1}{m+1}\binom{m}{k}^{-1} P[v,S]
    +\sum_{\substack{y\in S\subseteq G\cup\{y\}\\|S|=k}} \frac{1}{m+1}\binom{m}{k}^{-1} P[v,S]\\
    &=\sum_{\substack{S\subseteq G\\|S|=k}} \frac{m-k}{m+1} \cdot \frac{1}{m}\binom{m-1}{k}^{-1} P[v,S]
    +\sum_{\substack{S\subseteq G\\|S|=k-1}} \frac{k}{m+1}\cdot\frac{[x_y\in I_{v,y}]}{c_v(y)}\cdot \frac{1}{m}\binom{m-1}{k-1}^{-1} P[v,S]\\
&=\frac{m-k}{m+1}\cdot \phi(v,G,k)+
\frac{k}{m+1}\cdot \frac{[x_y\in I_{v,y}]}{c_v(y)}\cdot \phi(v,G,k-1).\qedhere
\end{align*}
\end{proof}
Now, given the recursive formula of Lemma~\ref{l:dp},
we are ready to prove Lemma~\ref{l:sum}.
\begin{proof}[Proof of Lemma~\ref{l:sum}]
Recall that $y\in U\setminus G\setminus F_v$ and thus
  $\frac{[x_y\in I_{v,y}]}{c_v(y)}=1$. By Lemma~\ref{l:dp}, we have:
\begin{align*}
\phi(v,G\cup \{y\}) &= \sum_{k=0}^{|G|+1}\phi(v,G\cup\{y\},k)\\
                  &=\sum_{k=0}^{|G|+1}\frac{|G|+1-k}{|G|+2}\cdot \phi(v,G,k)
                  +\sum_{k=0}^{|G|+1}\frac{k}{|G|+2}\cdot \phi(v,G,k-1)\\
                  &=\sum_{k=0}^{|G|}\frac{|G|+1-k}{|G|+2}\cdot \phi(v,G,k)
                  +\sum_{k=0}^{|G|}\frac{k+1}{|G|+2}\cdot \phi(v,G,k)\\
                  &=\sum_{k=0}^{|G|}\phi(v,G,k)\\
                  &=\phi(v,G).\qedhere
\end{align*}
\end{proof}
By Lemma~\ref{l:dp}, we can compute all the values of the form $\phi(v,G,\cdot)$
in $O(|G|)$ time given:
\begin{itemize}
  \item either the values ${\phi(v,G\setminus\{y^-\},\cdot)}$ for any $y^-\in G$,
  \item or the values $\phi(v,G\cup\{y^+\},\cdot)$ for any $y^+\in U\setminus G$.
\end{itemize}
To describe our basic algorithm, we need one more simple observation.
\begin{lemma}\label{l:lift}
  Let $v\in \tr$ be a non-root node and let ${Q\subseteq U\setminus \{z_v\}}$. Then, for all $k$,
  \begin{equation*}
    \phi(v,Q,k)=\phi(p_v,Q,k)\cdot\frac{c_v(z_v)}{c_{p_v}(z_v)}.
  \end{equation*}
\end{lemma}
\begin{proof}
  By Lemma~\ref{c:move_down}, the sum~(\ref{eq:dpdef}) defining $\phi(v,Q,k)$ can be obtained by multiplying every term in the sum defining $\phi(p_v,Q,k)$ by $\frac{c_v(z_v)}{c_{p_v}(z_v)}$.
\end{proof}

Now we are ready to describe the basic algorithm.
The algorithm will maintain a real vector $\Psi$:
\begin{equation*}
  \Psi(v,G)=(\Psi_k)_{k=0}^{|G|}=(\phi(v,G,k))_{k=0}^{|G|}
\end{equation*}
for some $v\in\tr$ and $G\subseteq U$.
We call this vector $\Psi$ a \emph{state}.
Note that $||\Psi||_1=\phi(v,G)$.
The algorithms will perform a number of operations on the state.

We will need the two following
vector functions
$\texttt{AddFeature}(\mathbf{b},\delta):\mathbb{R}^\ell\times \mathbb{R}\to \mathbb{R}^{\ell+1}$
and
\linebreak
$\texttt{DelFeature}(\mathbf{b},\delta):\mathbb{R}^{\ell}\times \mathbb{R}\to \mathbb{R}^{\ell-1}$.
Let us assume $\mathbf{b}:=(b_k)_{k=1}^\ell$ and for convenience put $b_{-1}=b_{\ell+1}=0$.
Then $\texttt{AddFeature}$ is defined:
\begin{equation*}
  (\texttt{AddFeature}(\mathbf{b},\delta))_k=\frac{\ell+1-k}{\ell+2}\cdot b_k+\frac{k}{\ell+2}\cdot \delta\cdot b_{k-1}.
\end{equation*}
The function $\texttt{DelFeature}$ is a reverse of $\texttt{AddFeature}$.
Formally, if $b_{-1}'=0$ and $(\texttt{DelFeature}(\mathbf{b},\delta))_k=b_k'$, then:
\begin{equation*}
  b_k'=\frac{\ell+1}{\ell}\left(b_k-\frac{k}{\ell+1}\cdot \delta\cdot b_{k-1}'\right).
\end{equation*}
Intuitively, the purpose of the functions $\texttt{AddFeature}$ and $\texttt{DelFeature}$
is to apply the dynamic programming transition of Lemma~\ref{l:dp}.
Defining these functions in such a general way will prove very useful
in the implementation of a faster algorithm (Section~\ref{a:faster}).
Clearly, both functions can be implemented in $O(\ell)$ time.

\begin{algorithm}[tb]
  \caption{Computing values $\phi_i$, $i\in U$, in $O(LD^2)$ time.}
\label{alg:traverse}
\textbf{procedure} $\texttt{Traverse}(v)$
\begin{algorithmic}[1]
\IF{$z_v\in F_{p_v}$}
  \STATE $\Psi:=\texttt{DelFeature}(\Psi,\delta_{v,z_v})$
\ENDIF
\STATE $\Psi:=\Psi\cdot \frac{c_v(z_v)}{c_{p_v}(z_v)}$
\STATE $\Psi:=\texttt{AddFeature}(\Psi,\delta_{v,z_v})$
\IF {$v\notin \lvs(\tr)$}
  \STATE $\texttt{Traverse}(a_v)$
  \STATE $\texttt{Traverse}(b_v)$
\ELSE
  \FOR{$i\in F_v$}
\STATE $\Psi:=\texttt{DelFeature}(\Psi,\delta_{v,i})$
  \STATE $\phi_i:=\phi_i+||\Psi||_1\cdot f(v)\cdot \left(\delta_{v,i}-1\right)$
\STATE $\Psi:=\texttt{AddFeature}(\Psi,\delta_{v,i})$
\ENDFOR
\ENDIF
\STATE $\Psi:=\texttt{DelFeature}(\Psi,\delta_{v,z_v})$
  \STATE $\Psi:=\Psi\cdot \frac{c_{p_v}(z_v)}{c_{v}(z_v)}$
\IF{$z_v\in F_{p_v}$}
  \STATE $\Psi:=\texttt{AddFeature}(\Psi,\delta_{v,z_v})$
\ENDIF
\end{algorithmic}

\textbf{function} $\texttt{TreeSHAP}$

\begin{algorithmic}[1]
  \STATE $\phi=(0,0,\ldots,0)\in \mathbb{R}^n$
  \STATE $\Psi:=0$
  \STATE $\texttt{Traverse}(a_\rho)$
  \STATE $\texttt{Traverse}(b_\rho)$
  \STATE \textbf{return} $\phi$
\end{algorithmic}
\end{algorithm}

The goal of a recursive procedure $\texttt{Traverse}(v)$ (see Algorithm~\ref{alg:traverse}) is to
iterate through all the states $\Psi(v,F_v)$, $v\in \tr[v]$.
At a leaf $l\in \lvs(\tr)$, the contributions
of the leaf $l$ to each relevant $\phi_i$ will be calculated.
Let us introduce one more notation (also used in the pseudocode of $\texttt{Traverse}$):
\begin{equation}
  \delta_{v,y}:=\frac{[x_v\in I_{v,y}]}{c_v(y)}.
\end{equation}

More precisely, when $\texttt{Traverse}(v)$ is called,
we guarantee that the current state is $\Psi(p_v,F_{p_v})$.
In particular, for $p_v=\rho$, $\Psi=(\phi(\rho,\emptyset,0))=(0)$.
The first step is to move the state to $\Psi(v,F_v)$.
This is done in three substeps:
\begin{enumerate}
  \item First, if $z_v\in F_{p_v}$, we perform
    \begin{equation*}
    \Psi:=\texttt{DelFeature}(\Psi,\delta_{v,z_v}).
    \end{equation*}
    Then, by Lemma~\ref{l:dp}, the state equals $\Psi(p_v,F_{p_v}\setminus \{z_v\})$.
  \item Next, we perform
    \begin{equation*}
    \Psi:=\Psi\cdot \frac{c_v(z_v)}{c_{p_v}(z_v)}.
    \end{equation*}
    Afterwards, since $z_v\notin F_{p_v}\setminus z_v$,
    by Lemma~\ref{l:lift}, the state
    equals $\Psi(v,F_{p_v}\setminus \{z_v\})$.
  \item Then, we perform
    \begin{equation*}
    \Psi:=\texttt{AddFeature}(\Psi,\delta_{v,z_v}).
    \end{equation*}
    which, again by Lemma~\ref{l:dp}, moves the state to
    $\Psi(v,F_v)$.
\end{enumerate}
If $v$ is non-leaf, $\texttt{Traverse}(a_v)$ and $\texttt{Traverse}(b_v)$
are called recursively to process the subtrees of $v$.
Note that the required state invariant is satisfied at the beginning
of these calls.
After the recursive calls return, we move the state back
to $\Psi(p_v,F_{p_v})$ by reversing the steps 1-3.

If $v$ is a leaf, for each $i\in F_v$ we do the following:
\begin{enumerate}
  \item We move the state to $\Psi(v,F_v\setminus\{i\})$ by performing
    $\Psi:=\texttt{DelFeature}(\Psi,\delta_{v,i})$.
  \item We compute $||\Psi||_1$ to obtain $\phi(v,F_v\setminus\{i\})$
    and add the contribution $\phi(v,F_v\setminus\{i\})\cdot f(v)\cdot \left(\frac{[x_i\in I_{v,i}]}{c_v(i)}-1\right)$ to $\phi_i$.
  \item We move the state back to $\Psi(v,F_v)$ by running
    $\Psi:=\texttt{AddFeature}(\Psi,\delta_{v,i})$.
\end{enumerate}

Clearly, the computation at each non-leaf node $v\in\tr$
costs $O(|F_v|)=O(D)$ time.
On the other hand, for a leaf node $l$, we spend
$O(|F_l|^2)=O(D^2)$ time.
As a result, the total running time of the algorithm
is $O(LD^2)$.

For clarity of the pseudocode, we assume that
$\delta_{v,i}$ for any $i\in U$ can be computed in $O(1)$ time
while inside the call $\texttt{Traverse}(v)$.
We now justify this assumption.
In the beginning of Section~\ref{a:simple} we argued how
the values $I_{v,y}$ and $c_v(y)$, $y\in U$, can differ from the corresponding values
$I_{p_v,y}$ and $c_{p_v}(y)$ only for a single $y=z_v$.
Note that the call $\texttt{Traverse}(v)$ only requires
values of the form $I_{v,\cdot}$, $c_v(\cdot)$. As a result,
both these sets of values can be maintained using global
arrays $\texttt{I}:U\to \mathbb{R}^2$ and $\texttt{c}:U\to \mathbb{R}$.
When the call $\texttt{Traverse}(v)$ starts,
we only need to update $\texttt{I}[z_v]$ and $\texttt{c}[z_v]$ to
reflect the difference between the $I_{v,\cdot},c_v(\cdot)$
and $I_{p_v,\cdot},c_{p_v}(\cdot)$.
When the call ends, we revert that change to the arrays $\texttt{I},\texttt{c}$.

\subsection{Banzhaf Values}\label{a:banzhaf}
The same Algorithm~\ref{alg:traverse} can be used to compute
Banzhaf values. We only need to change the definition
of a state $\Psi(v,G)$.

Instead of values $\phi(v,G,k)$, we base our dynamic programming computation
on values $\beta(v,G)$ defined as follows.
\begin{equation}\label{eq:dp2def}
\beta(v,G):=\frac{1}{2^{|G|}}\sum_{S\subseteq G} y(v,G).
\end{equation}
The third coordinate is dropped since the coefficient
of the summands in~(\ref{eq:dp2def}) do not depend
on the size of the set $S$.
We now show an analogue of Lemma~\ref{l:dp} for Banzhaf values.
\begin{lemma}\label{l:dp-ban}
  Let $v\in \tr$ and $G\subseteq U$. Let $y\in U\setminus G$. Then:
  \begin{equation*}
    \beta(v,G\cup\{y\})=\frac{1}{2}\left(1+\frac{[x_f\in I_{v,f}]}{c_v(f)}\right)\beta(v,G).
  \end{equation*}
\end{lemma}
\begin{proof}
  Let $m=|G|$. By Lemma~\ref{c:add_feature} we have:
  \begin{align*}\beta(v,G\cup\{y\})&=\sum_{S\subseteq G\cup\{y\}}\frac{1}{2^{m+1}}P[v,S]\\
    &=\sum_{S\subseteq G} \frac{1}{2^{m+1}}P[v,S]
    +\sum_{y\in S\subseteq G\cup\{y\}} \frac{1}{2^{m+1}}P[v,S]\\
    &=\sum_{S\subseteq G} \frac{1}{2} \cdot \frac{1}{2^m} P[v,S]
    +\sum_{S\subseteq G} \frac{1}{2}\cdot\frac{[x_y\in I_{v,y}]}{c_v(y)}\cdot \frac{1}{2^m} P[v,S]\\
&=\frac{1}{2}\left(1+\frac{[x_y\in I_{v,y}]}{c_v(y)}\right)\beta(v,G). \qedhere
\end{align*} 
\end{proof}
By Lemma~\ref{l:dp-ban}, it also follows
that for $y\notin F_v$, we have $\beta(v,G\cup \{y\})=\beta(v,G)$,
which gives a Banzhaf analogue of Lemma~\ref{l:sum}.
Similarly, the proof of Lemma~\ref{l:lift}
carries on to Banzhaf values and
we get that for $z_v\notin Q$ we have
\begin{equation*}
  \beta(v,Q)=\beta(p_v,Q)\cdot\frac{c_v(z_v)}{c_{p_v}(z_v)}.
\end{equation*}
Finally, we can analogously extract the individual
leaf contributions to the Banzhaf values and
get exactly the same formula:
\begin{equation*}
  \beta_i=\sum_{\substack{l\in\lvs(\tr)\\i\in F_l}}f(l)\cdot \left(\frac{[x_i\in I_{l,i}]}{c_l(i)}-1\right)\cdot \beta(l,F_l\setminus\{i\}).
\end{equation*}
The above discussion shows that precisely
the same Algorithm~\ref{alg:traverse} can be used
if we modify the procedures $\texttt{AddFeature}$,
$\texttt{DelFeature}$ to operate on the Banzhaf
state $\Psi(v,G):=\beta(v,G)$ in place of the Shapley state,
i.e., to work according to Lemma~\ref{l:dp-ban}
instead of Lemma~\ref{l:dp}.
Observe that the implementation of these operations
is simpler for Banzhaf values
and all can be implemented in constant time.

Therefore, computing all values $\beta(l,F_l)$
for $l\in\lvs(\tr)$ takes only $O(L)$ time.
Computing all the values $\beta(l,F_l\setminus\{i\})$
(and thus individual leaf contributions)
for
$l\in\lvs(\tr)$ and $i\in F_l$, however,
still takes $O(LD)$ time since
even the number of such pairs $(l,i)$ can be $\Theta(LD)$.
Nevertheless, this is already
a factor-$D$ speed-up over the basic
algorithm for computing Shapley values.

\subsection{Handling Multiple Trees}\label{s:multiple}
If the tree ensemble consists of more than one tree, i.e., $T>1$,
the output of the model is taken to be the average
of the outputs of the individual trees.
Since so far $g(S)$ was meant to approximate
$\EX[f(x_S,X_{\bar{S}})]$ for a single tree,
by linearity of expectations, we can
redefine $g(S)$ to be the average $g(S)$ over the individual trees.

In the implementation, all we need to do is
to run Algorithm~\ref{alg:traverse} for each of the individual
trees, but we have to initialize $\phi=(0,\ldots,0)\in \mathbb{R}^n$
only once; we should not do it for each individual tree.
After we are finished, the values $\phi_i$ should be
divided by $T$.
Note that, since the $O(n)$ term in the $O(LD^2+n)$ bound
came only from initializing the vector $\phi$,
the running time for $T$ trees is $O(TLD^2+n)$.

\section{Faster Algorithm for Tree Ensembles}\label{a:faster}
Recall Equation~(\ref{eq:leaf-contrib}) that
reduced the computation of all $\phi_i$
to computing individual $O(LD)$ leaf contributions
\begin{equation*}
  \left(\frac{[x_i\in I_{l,i}]}{c_l(i)}-1\right)\cdot f(l)\cdot \phi(l,F_l\setminus\{i\})
\end{equation*}

Let $\lvs_v$ denote the set of leaves $l\in \tr[v]$
such that $v$ is the nearest ancestor $w$ of $l$
with $z_w=z_v$. Note that $\lvs_v$ might not
contain all the leaves in $\tr[v]$ if some
descendant of $v$ satisfies $z_w=z_v$.
What is important, for all $l\in \lvs_v$ we have
\begin{equation*}
  \left(\frac{[x_{z_v}\in I_{l,{z_v}}]}{c_l(z_v)}-1\right)=\left(\frac{[x_{z_v}\in I_{v,{z_v}}]}{c_v(z_v)}-1\right).
\end{equation*}

Moreover, note that the sets $\{\lvs_v:v\in\tr, z_v=i\}$
form a partition of the set
$\{l\in\lvs(\tr):i\in F_l\}$.
Let us also set:
\begin{equation*}
  \Phi^-(v):=\sum_{l\in \lvs_v}f(l)\cdot \phi(l,F_l\setminus \{z_v\}).
\end{equation*}
Therefore, we can rewrite~(\ref{eq:leaf-contrib}) as follows:
\begin{align*}
  \phi_i&=\sum_{\substack{l\in\lvs(\tr)\\i\in F_l}}f(l)\cdot \left(\frac{[x_i\in I_{l,i}]}{c_l(i)}-1\right)\cdot \phi(l,F_l\setminus\{i\})\\
  &=\sum_{\substack{v\in\tr\\z_v=i}}\sum_{l\in\lvs_v}f(l)\cdot \left(\frac{[x_i\in I_{l,i}]}{c_l(i)}-1\right)\cdot \phi(l,F_l\setminus\{i\})\\
  &=\sum_{\substack{v\in\tr\\z_v=i}} \left(\frac{[x_i\in I_{v,i}]}{c_v(i)}-1\right)\sum_{l\in\lvs_v}f(l)\cdot \phi(l,F_l\setminus\{i\})\\
  &=\sum_{\substack{v\in\tr\\z_v=i}} \left(\frac{[x_i\in I_{v,i}]}{c_v(i)}-1\right)\cdot \Phi^-(v).
\end{align*}
Note that the above derivation provides an $O(L)$-time
reduction of computing all $\phi_i$ to computing
all values $\Phi^-(v)$.
The remaining part of this section is thus devoted
to computing the values $\Phi^-(v)$, $v\in \tr$.

Before we continue, we need to make a subtle change
to our previous algorithm.
Both the correctness and efficiency of our
algorithm will crucially rely on the assumption
that all sets $F_l$ for $l\in\lvs(\tr)$ have the same size.
This
could be ensured, for example, by extending all smaller $F_l$ with $D-|F_l|$
distinct dummy features that do not appear in $F_l$ -- recall
that, by Lemma~\ref{l:sum}, adding dummy features does not change $\phi(v,G)$,
for any $G\subseteq U$, so it does not influence our results.
Unfortunately, adding a dummy feature to $F_l$
by simply using Lemma~\ref{l:dp} costs $\Theta(D)$ time.
Therefore, if $T$ was very
unbalanced, padding all $F_l$ could cost as much as $\Theta(LD^2)$ time.

Instead, let $q_1,\ldots,q_D$ be distinct artificial features \emph{not} appearing
in the nodes of $\tr$. For \emph{all} $v\in \tr$ let us define
\begin{equation*}
  F^*_v=F_v\cup \{q_1,\ldots, q_{D-|F_v|}\}.
\end{equation*}
Observe that then $F^*_\rho=\{q_1,\ldots,q_D\}$ for the root $\rho$ of $\tr$, and for each
non-root $v$ we have
\begin{equation*}
  F^*_v=\begin{cases}F^*_{p_v}&\text{ if }z_v\in F_{p_v}\\
    F^*_{p_v}\setminus\{q_{D-|F_{p_v}|}\}\cup \{z_v\}&\text{ otherwise.}
\end{cases}
\end{equation*}
With sets $F^*_v$ defined like this, $v\in\tr$, by Lemma~\ref{l:sum}, we have:
\begin{equation*}
  \phi(v,F_v)=\phi(v,F_v^*),
\end{equation*}
and consequently:
\begin{equation*}
  \Phi^-(v)=\sum_{l\in \lvs_v}f(l)\cdot \phi(l,F_l^*\setminus\{z_v\}).
\end{equation*}

We modify the basic algorithm computing all the states
$\Psi(v,F_v)$ so that it computes all the states $\Psi(v,F_v^*)$.
It is very easy to change it to achieve that.
First of all, the initial state $\Psi(\rho,F^*_\rho)$
is initialized in $O(D^2)=O(LD)$ time by applying
Lemma~\ref{l:dp} $D$ times.
Recall that the first step of $\texttt{Traverse}$ is
to perform $\Psi:=\texttt{DelFeature}(\Psi,\delta_{v,z_v})$ if $z_v\in F_{p_v}$.
Now, all we need to do is to add an extra condition that,
if $z_v\notin F_{p_v}$, we also perform $\Psi:=\texttt{DelFeature}(\Psi,\delta_{v,q_{D-|F_{p_v}|}})$
to ``remove'' a dummy feature from the state.
Of course, this change has to be reflected also
in the code that undoes state manipulation
after the recursive calls return.
See Algorithm~\ref{alg:traverse-fast}.

\begin{algorithm}[ht!]
  \caption{Computing values $\phi_i$, $i\in U$, in $O(LD)$ time.}
\label{alg:traverse-fast}
\textbf{procedure} $\texttt{TraverseFast}(v)$
\begin{algorithmic}[1]
\IF{$z_v\in F_{p_v}$}
  \STATE $\Psi:=\texttt{DelFeature}(\Psi,\delta_{v,z_v})$.
  \STATE $\texttt{Push}(H[z_v],v)$.
\ELSE
  \STATE $\Psi:=\texttt{DelFeature}(\Psi,\delta_{v,q_{D-|F_{p_v}|}})$.
\ENDIF
\STATE $\Psi:=\Psi\cdot \frac{c_v(z_v)}{c_{p_v}(z_v)}$.
\STATE $\Psi:=\texttt{AddFeature}(\Psi,\delta_{v,z_v})$.
\IF {$v\notin \lvs(\tr)$}
  \STATE $\texttt{TraverseFast}(a_v)$.
  \STATE $\texttt{TraverseFast}(b_v)$.
  \STATE $S(v):=S(a_v)+S(b_v)$.
\ELSE
  \STATE $S(v):=f(v)\cdot \Psi$.
\ENDIF
  \STATE $\Gamma(v):=S(v)$
\WHILE{$\texttt{Top}(H[z_v])\neq v$}
  \STATE $w:=\texttt{Pop}(H[z_v])$
  \STATE $\Gamma(v):=\Gamma(v)-S(w)$
\ENDWHILE
\STATE $\Gamma^-(v):=\texttt{DelFeature}(\Gamma(v),\delta_{v,z_v})$.
\STATE $\phi_i:=\phi_i+||\Gamma^-(v)||_1\cdot \left(\delta_{v,{z_v}}-1\right)$.
\STATE $\Psi:=\texttt{DelFeature}(\Psi,\delta_{v,z_v})$.
  \STATE $\Psi:=\Psi\cdot \frac{c_{p_v}(z_v)}{c_{v}(z_v)}$.
\IF{$z_v\in F_{p_v}$}
  \STATE $\Psi:=\texttt{AddFeature}(\Psi,\delta_{v,z_v})$.
\ELSE
  \STATE $\Psi:=\texttt{AddFeature}(\Psi,\delta_{v,q_{D-|F_{p_v}|}})$.
\ENDIF
\end{algorithmic}

\textbf{function} $\texttt{TreeSHAPFast}$

\begin{algorithmic}[1]
  \STATE $\phi=(0,0,\ldots,0)\in \mathbb{R}^n$.
  \STATE $\Psi:=0$.
  \FOR{$i=1$\textbf{ to }$D$}
    \STATE $\Psi:=\texttt{AddFeature}(\Psi,0)$.
  \ENDFOR
  \FOR{$y\in U$}
    \STATE $H[y]:=\emptyset$
  \ENDFOR
  \STATE $\texttt{Traverse}(a_\rho)$.
  \STATE $\texttt{Traverse}(b_\rho)$.
  \STATE \textbf{return} $\phi$
\end{algorithmic}
\end{algorithm}

The faster algorithm (implemented using a recursive
procedure $\texttt{TraverseFast}$) avoids computing, for each $l\in \lvs(\tr)$,
the states $\Psi(l,F_l^*\setminus\{i\})$ for all $i\in F_l$,
which could take as much as $\Theta(LD^2)$ time in the basic algorithm.
However, to achieve speed-up, the faster algorithm
uses additional ``bottom-up'' steps (i.e., following the recursive calls).
These steps compute
some auxiliary data that we describe next.

For all $v\in\tr$, let us define
vectors $\Gamma(v),S(v)\in\mathbb{R}^D$, such that for $k=0,\ldots,D$:
\begin{align*}
  \Gamma(v)_k&=\sum_{l\in \lvs_v}f(l)\cdot \phi(l,F_l^*,k),\\
  S(v)_k&=\sum_{l\in \lvs(T[v])}f(l)\cdot \phi(l,F_l^*,k).
\end{align*}

\begin{lemma}
  The vectors $S(v)$ and $\Gamma(v)$ for all $v\in\tr$ and $k\in\{0,\ldots,D\}$ can be computed using additional bottom-up steps in $\texttt{TraverseFast}$
  that cost $O(LD)$ extra time in total.
\end{lemma}
\begin{proof}
First note that
  for $l\in \lvs(\tr)$ we have $S(v,k)={f(l)\cdot \phi(l,F^*_l,k)}$
  and $\phi(l,F^*_l,k)\in \Psi(l,F^*_l)$ which is a state
  that the basic procedure $\texttt{Traverse}$ computes.

  For all non-leaf nodes $v\in\tr$ we in turn have
  \begin{equation*}
    S(v)=S(a_v)+S(b_v),
  \end{equation*}
  and thus each of such $O(LD)$ values $S(v)_k$ can be computed
  in constant time after the recursive calls return.

  Given the values $S(v)_k$, it is not very hard to obtain values
  $\Gamma(v)_k$. Let $Q_v$ be the set of nodes $w\in \tr[v]$
  such that $z_w=z_v$ and $v$ is the nearest ancestor
  of $w$ satisfying $z_w=z_v$.
  Note that we have
  \begin{equation*}
    \lvs_v = \lvs(\tr[v])\setminus \left(\bigcup_{w\in Q_v} \lvs(\tr[w])\right),
  \end{equation*}
  and thus
  \begin{equation*}
    \Gamma(v)=S(v)-\sum_{w\in Q_v} S(w).
  \end{equation*}
  Observe that the total size of sets $Q_v$ (over all $v\in\tr$) is $O(L)$,
  so if we are allowed to iterate through $Q_v$ whenever we wish to compute
  $\Gamma(v)$, the computation of  $\Gamma(v)$
  takes $O(LD)$ time as well.
  Let $g_{w,j}$ denote the nearest ancestor of $w\in \tr$
  with $z_w=j$.
  One way to enable iterating through
  $Q_v$ when $v$ is processed bottom-up,
  is to maintain, for each feature $j\in U$,
  a global stack $H[j]$ containing
  all the nodes $w$ such that $z_w=j$
  and $\texttt{TraverseFast}(g_{w,j})$
  has not yet completed.
  The stack elements are sorted using
  the pre-order of the nodes of $v$, so that
  the node $w$ with the highest pre-order
  is at the top of $H[z_w]$.
  The stack can be updated in $O(1)$ time
  whenever a recursive call starts.
  Observe that $v\in H[z_v]$ when $\texttt{TraverseFast}(v)$
  has started but has not yet finished.
  Now, given $H[z_v]$, it is enough to note
  that $Q_v$ equals precisely the set of elements
  of $H[z_v]$ that lie higher than $v$.
  Thus, one can indeed iterate through $Q_v$
  in $O(|Q_v|)$ time as desired.
  Moreover, $Q_v$ constitutes precisely the
  set of elements that have to be popped from the
  stack $H[z_v]$ when $\texttt{TraverseFast}(v)$ returns.
  The asymptotic cost of popping stack elements can charged
  to the corresponding pushes and thus can be neglected.
\end{proof}

Finally, we show that, roughly speaking, the same recursive
relation as in Lemma~\ref{l:dp} can be
applied to the values $\Gamma(v)_k$
in order to obtain the required values $\Phi^-(v)$.

\begin{lemma}\label{l:dp2}
Let $\Gamma^-(v)_k=\sum_{l\in\lvs_v}f(l)\cdot \phi(l,F^*_l\setminus\{z_v\},k)$.
  For any $l\in \lvs_v$ and $k\in \{0,\ldots,D\}$ we have:
\begin{align*}
  &\Gamma(v)_k=\frac{D-k}{D+1}\cdot \Gamma^-(v)_k + \frac{k[x_{z_v}\in I_{v,z_v}]}{(D+1)c_v(z_v)}\cdot \Gamma^-(v)_{k-1}.
\end{align*}

\end{lemma}
\begin{proof}
  Since $z_v\in F_l^*$ for all $l\in \lvs_v$, by Lemma~\ref{l:dp},
  we have
  \begin{align*}
    \phi(l,F^*_l,k)&=\frac{D-k}{D+1}\biggl(\phi(l,F^*_l\setminus\{z_v\},k)\left.-\frac{k[x_{z_v}\in I_{l,z_v}]}{(D+1)c_l(z_v)}\phi(l,F^*_l\setminus\{z_v\},k-1)\right)\\
 &=\frac{D-k}{D+1}\biggl(\phi(l,F^*_l\setminus\{z_v\},k)\left.-\frac{k[x_{z_v}\in I_{v,z_v}]}{(D+1)c_v(z_v)}\phi(l,F^*_l\setminus\{z_v\},k-1)\right).
  \end{align*}
We obtain the desired equality by summing the above
through all $l\in\lvs_v$.
\end{proof}

For a convenient implementation, note that the vectors
$\Gamma^-(v)_k$ are obtained from the corresponding vectors
$\Gamma(v)_k$ in exactly the same way as $\Psi(v,F_v\setminus\{z_v\})$
was obtained from $\Psi(v,F_v)$
in the basic algorithm.
Therefore, the functions $\texttt{AddFeature}$, $\texttt{DelFeature}$
can be reused.
More specifically, the following hold:
\begin{align*}
  \Gamma(v)=\texttt{AddFeature}(\Gamma^-(v),\delta_{v,z_v}),\qquad
  \Gamma^-(v)=\texttt{DelFeature}(\Gamma(v),\delta_{v,z_v}).
\end{align*}

\section{Other Feature Attribution Methods}
\label{section:other}
Feature importance values summarize a complicated ensemble model and
provide insight into what features drive the model's prediction. There can be
various types of explanation methods to compute such values: model-dependent
or model-agnostic methods, global or local explanation methods.

\paragraph{Explanation methods for trees:} Global feature importance values are
computed for an entire dataset in mainly three different ways. The basic global
approach, \emph{Split Count}, is to count the number of times a feature is used
for splitting \cite{ChenG16}. However, this method fails to account for the impacts
of different splits. The \emph{Gain} approach to feature importance \cite{Breiman}
is to attribute the reduction of loss contributed by each split in each decision tree
and it is widely used as the basis for feature selection methods~\cite{Chebrolu2005,HuynhThu2010InferringRN,Sandri2008ABC}. Another commonly used approach,
\emph{Permutation}, is to randomly permute the data column corresponding to a
feature in the test set and observe the change in the model's loss \cite{Breiman2004}.
If the model is heavily dependent on the feature then permuting it should create a
large increase in the model's loss.

These approaches are designed to estimate the global importance of a feature
over an entire dataset, so they are not directly applicable to local explanations
that are specific to each prediction. Local explanation methods for computing
feature importance values for a single prediction are not well studied for trees.
Only a couple of tree-specific local explanation methods were known previously.
One is to just report the decision path, which is not useful for large tree ensembles.
The other one is by Saabas~\cite{Saabas} which is a heuristic method that measures
the difference in the model's expected output. The Saabas method explains a prediction
by following the decision path of the current input and attributing the differences in the
expected output of the model to each of the features along the path. The expected value
of every node in the tree is the average of the model output over the training samples
going through that node. For explaining an ensemble model made up of a sum of many
trees, the Saabas value for the ensemble is defined as the sum of the values for each tree.

As noted in \cite{lundberg2018consistent}, the feature importance values from the gain,
split count, and Saabas methods are all inconsistent i.e., a model can be modified so that
it relies more on a given feature, yet the importance assigned to that feature decreases.

\paragraph{Model-agnostic methods:} One of the most common local explanation
methods in deep learning literature is to take the gradient of the model's output with
respect to its inputs at the current sample or multiplying the gradient times the value
of the input features. As depending entirely on the gradient of the model at a single
point can often be misleading \cite{Shrikumar2016NotJA} various other methods
have also been proposed~\cite{Springenberg2015StrivingFS,Zeiler2014VisualizingAU,Bach2015OnPE,Shrikumar2016NotJA,Kindermans2018LearningHT,Ancona2018TowardsBU}.

Model-agnostic methods on the other hand make no assumptions about the internal
structure of the model and depend on the relationship between changes in the model
inputs and model outputs. This is achieved by training a global mimic model to approximate
the original model, then locally explaining the mimic model \cite{Baehrens10a,Plumb18}.
Alternatively, the mimic model can be fit into the original model locally for each prediction.
In the LIME method~\cite{Ribeiro2016WhySI} the coefficients are used as an explanation
for a local linear mimic model. In Anchors~\cite{ribeiro2018anchors} the rules are used
as the explanation for a local decision rule mimic model.

Recently, several methods for the local explanation of model predictions (such as
LIME~\cite{Ribeiro2016WhySI}, DeepLIFT \cite{Shrikumar2016NotJA,shrikumar17a},
Layer-wise Relevance Propagation~\cite{Bach2015OnPE}, and three methods from
cooperative game theory: Shapley regression values \cite{Lipovetsky2001AnalysisOR},
Shapley sampling values~\cite{StrumbeljK14}, and Quantitative Input
Influence~\cite{Datta2016AlgorithmicTV}) are unified into a single class of
\emph{additive feature attribution methods}~\cite{Lundberg2017}. This class contains
methods that explain a model's output as a sum of real values attributed to each input
feature. It is of particular interest as there is a unique optimal explanation approach in
the class that satisfies three desirable properties: local accuracy, missingness, and
consistency~\cite{Roth,Shapley53}. \emph{Local accuracy} (also called \emph{Efficiency}
or \emph{Completeness}) means that the sum of the feature attributions is equal to the
output of the function we want to explain. \emph{Missingness} (also called \emph{Sensitivity},
or \emph{Null-player axiom}) means that missing features are given no importance and
\emph{Consistency} (also called \emph{Monotonicity}) means that if a feature has a larger
impact on the model after a change then the attribution assigned to that feature can only increase.

One can use model-agnostic local explanation methods to explain tree models however
their dependence on post-hoc modeling of an arbitrary function can make them slow or
might suffer from sampling variability for models with many input features~\cite{Lundberg2020}.
Although such methods are often practical for individual explanations, but can quickly
become impractical for explaining entire datasets.

\section{Conclusions}
The contribution of this paper is twofold. First, we have developed new and more efficient algorithms for computing feature importance measures for tree ensemble
models that are based on Banzhaf and Shapley values. These results improve the running time of previously known methods.
Second, we present the first extensive comparison between Shapley and Banzhaf values. We observe that both methods deliver explanations of essentially the same strength by returning
almost the same ordering of features. However, these experimental results indicate that Banzhaf values have several important advantages over Shapley values, i.e.,
allow for faster algorithms as well as these algorithms make much lower numerical errors. As argued in the introduction, one should not base his decision on which
index to use based on axiomatic characterizations, and other aspects should be considered. In particular, our study has delivered two important arguments for the
usage of Banzhaf values.

\clearpage
\bibliographystyle{abbrv}
\bibliography{references}

\end{document}